%% file: main.tex
\documentclass[pmlr]{jmlr}

\usepackage{booktabs}
\usepackage{comment}
\usepackage{wrapfig}
\usepackage{algorithm}
 
 \usepackage[]{todonotes}
 \usepackage{soul}

\expandafter\let\csname enumerate*\endcsname\undefined
\expandafter\let\csname endenumerate*\endcsname\undefined
\usepackage[inline]{enumitem}
\setlist{topsep=1ex,parsep=.5ex,itemsep=0pt}

\jmlrvolume{}
\jmlryear{}
\jmlrworkshop{}
\jmlrpages{}
\jmlrproceedings{}{}

\usepackage[utf8]{inputenc}
\usepackage{amsmath}
\usepackage{bbm}
\usepackage{amsfonts}
\usepackage{array}
\usepackage{booktabs}
\usepackage{ifthen}
\usepackage{graphicx}
\usepackage{xspace}

\usepackage{tikz}
\usetikzlibrary{arrows,automata,positioning}

\input{defns}

\title{An Algebraic Approach to Learning and Grounding}
\author{
\Name{Johanna Björklund\nametag{\thanks{This work was partially supported by the Swedish Research Council and the Wallenberg Autonomous Systems, AI, and Software Program}}} \Email{johanna@cs.umu.se}\\ 
\and 
\Name{Adam {Dahlgren Lindström}} \Email{dali@cs.umu.se}\\ 
\and 
\Name{Frank Drewes\nametag{\thanks{Authors are given in alphabetical order}}} \Email{drewes@cs.umu.se}\\
\addr Dept.~Computing Science, Umeå University}

\begin{document}

\maketitle

\begin{abstract}

We consider the problem of learning the semantics of composite algebraic expressions from examples. The outcome is a versatile framework for studying learning tasks that can be put into the following abstract form: The input is a partial algebra $\alg$ and a finite set of examples $(\varphi_1, O_1), (\varphi_2, O_2), \ldots$, each consisting of an algebraic term $\varphi_i$ and a set of objects~$O_i$. The objective is to simultaneously fill in the missing algebraic operations in $\alg$ and ground the variables of every $\varphi_i$ in $O_i$, so that the combined value of the terms is optimised. We demonstrate the applicability of this framework through case studies in grammatical inference, picture-language learning, and the grounding of logic scene descriptions. 
\end{abstract}

\section{Introduction}
Modern natural language processing literature, and deep learning-based methods in particular, is critiqued for using terms like learning, grounding, meaning, and understanding loosely without clear definitions~\citep{bender2020climbing}.
Large language models do not learn representations that are sufficient for the applications they are used for, resulting in, e.g., perpetuation of discriminatory bias~\citep{bender2021dangers}.
Successful representational systems are usually modular, in that they describe complex concepts by combining representations for simpler ones. Examples include first-order predicate logic, where predicates, variables, and connectives are assembled into logic expressions, or natural language (NL), where lexical and grammatical constructs are the building blocks of sentences, which in turn combine to paragraphs, and so on. 
 When learning the components of a representational system from real-world observations, we seldom observe instances of a concept in isolation. Even in picture books for children, an image that is intended to convey the notion of a shape may also be understood to stand for, e.g., a colour or a size, so the child may need several observations to grasp the concept.  We are therefore interested in exploring the conditions under which composite semantics can be learnt by example. 

To this end, we introduce the notion of a \emph{template algebra}. This is an algebra $\alg$ that is defined with respect to a set of typed operator symbols $\Sigma$. The algebra is non-standard in that (i) there is a dedicated \emph{evaluation type} $\tau$ whose domain in $\alg$ is a linearly ordered set, and (ii) a subset of the operators for the symbols in $\Sigma$ have been left undefined. Given a set of terms over $\Sigma$ of type $\tau$, the learning task is to choose the missing operators (from families of candidate ones) so that the combined value of the terms is optimal. 

To integrate domains that are not finitely generated, the terms may contain variables which are assigned values from a target domain during evaluation. With this addition, a learning example is a pair $(\varphi, O)$, where $\varphi$ is a $\Sigma$-term with variables, and $O$ is a set of suitable typed elements that can be substituted for the variables. Again, the goal is to find an optimal solution over the entire set of examples, but the task has become more complicated, as the learning algorithm must simultaneously decide what object to map to what variable, and what operators to choose for those missing. Coming back to the picture book example, this correspond to simultaneously inferring what `red' means in a phrase like ``Look at the red ball'' and what object in the picture constitutes the red ball.  

This second subtask, which can be described as linking abstract representations to target domains, is known as \emph{grounding}. Other examples of grounding are connecting symbols on a map to features of the terrain, and to link words in a natural-language search query to tables, columns, and threshold values in a database. Also when training is done and the algebra complete, grounding can still remain challenging for unseen examples. However, it is likely that the more accurately the system has learnt the target operations, the better it is positioned also to solve the remaining problem of grounding, because it can more efficiently limit the search space when mapping objects to variables. 

The purpose of our framework is to provide flexible means for setting up formally defined symbolic, neural, or neural-symbolic tasks for learning from examples (where the emphasis is on tasks involving neural components). Moreover, this is done in such a way that corpora -- sets of examples of the type defined by the task -- can be generated automatically using grammars and automata. In summary, our approach provides a unified way to
\begin{itemize}
\item formalize learning tasks that involve mixtures of symbolic and neural components,
\item study the problem of grounding as an integral part of other learning tasks, and
\item automatically generate sets of examples for studying proposed learning algorithms.
\end{itemize}

To demonstrate the framework's ability to support a variety of classical and modern learning approaches, we discuss instantiations of the framework in three learning scenarios. The first 
is the grammatical inference of regular string language from characteristic sets. The second considers picture-language learning and involves the estimation of affine transformations with neural networks. The final instantiation serves to ground the variables of a logic scene description in geometric objects, by training neural networks to recognise atomic properties of the objects such as colour, shape, and size. This ability to integrate finite and infinite-state methods makes the proposed framework a contribution to the field of neuro-symbolic reasoning.

\subsection{Related work}
\label{sec:related}
Learning and grounding is addressed in symbolic, neural, and neuro-symbolic settings throughout the literature.
Our framework provides a way to systematically study the different methods and how they approach the grounding problem.
In this section we outline important work in the respective fields.

Visual grounding has been addressed in neuro-symbolic machine learning in various ways.
Inductive Logic Programming (ILP, \citet{muggleton1994inductive}) is a form of symbolic learning and reasoning with first-order clauses for learning relations in data.
Background knowledge is used together with examples to induce a hypothesis, in the form of a logic program, describing positive and negative examples in given data.
ILP is applied to, e.g., scientific discovery, robotics, program analysis~\citep{cropper2022inductive}, with neuro-symbolic approaches to machine vision problems~\citep{dai2015logical,varghese2021human}.

Probabilistic logic programming (PLP~\citet{dantsin1992probabilistic,ng1992probabilistic}), provides reasoning under uncertainty, with methods such as~\cite{de2007problog}.
\cite{raedt2008probabilistic} provides an overview of the combination of PLP and ILP, Probabilistic ILP.
Neuro-symbolic PLP, such as DeepProbLog, introduces predicates that are realised by neural networks to be trained, as a way to ground and reason about visual concepts~\citep{ManhaeveEtAl:2018,weber2019nlprolog,winters2021deepstochlog}.
In the Neuro-Symbolic Concept Learner by \citet{Mao2019NeuroSymbolic}, programs are taught to perform visual question answering by combining a neural perception module, a semantic parsing module to construct programs from language, and a program executor which bridges these components to provide answers.
The idea of Neural State Machines by \citet{hudson2019learning} builds on probabilistic graphs to learn and reason about concepts for visual question answering.
Logic Tensor Networks~\citep{serafini2016logic,badreddine2022logic} have been used for semantic image interpretation as a partial grounding problem~\citep{donadello2017logic}.
\citet{garcez2019neural} gives an overview of how learning and reasoning can be integrated, and how logical representations are used for that purpose.

Answer Set Programming (ASP) is another approach to knowledge representation and reasoning based on searching for stable models, answer sets, as solutions to given logic programs~\citep{niemela1999logic,marek1999stable,lifschitz2002answer}. 
ASP is used in decision support systems, planning robotics, and many other domains~\citep{erdem2016applications}, with neuro-symbolic examples in reasoning about objects in video~\citep{suchan2018visual}.

Finally, grounding in language and vision with neural networks is extensively researched for tasks such as Visual Question Answering (VQA)~\citep{antol2015vqa,fukui-etal-2016-multimodal,chaplot2018gated}.
Neuro-symbolic approaches to VQA include~\citep{yi2018neural}.

The above domains all have different perspectives on learning and grounding, and there are overlaps, such as between ILP and PLP.
We can formalise the described methods and the problems they try to solve, to examine the similarities, challenges, and opportunities.

\subsection{Outline}
The paper is outlined as follows: Section~\ref{sec:related} reports on related work on algebraic approaches and neuro-symbolic learning. Section~\ref{sec:preliminaries} recalls fundamental concepts and results, and Section~\ref{sec:algebraic} gives introduces the algebraic framework. Section~\ref{sec:applications} discusses the framework's application to learning and grounding of regular string languages, picture languages, and logic scene descriptions. Section~\ref{sec:future} provides  a summary and directions for future work. 

\section{Preliminaries}
\label{sec:preliminaries}
We begin by recalling some useful definitions and notations from literature. 

\subsection{Sets, numbers, and relations}
The set of natural numbers (including $0$) is denoted by $\nat$, and for $n\in\nat$, we write  $[n]$ as shorthand for $\{1,\dots,n\}$. In particular, $[0]=\emptyset$. For a set $S$, we denote by $\pow{S}$ the powerset of $S$. The domain of a partial function $\delta$ is denoted by $\domain(\delta)$.

Let $\sim$ be an equivalence relation on a superset of $S$.  The \emph{equivalence class} of~$s$ in~$S$ with respect to~$\sim$ is the set
$[s]^{S}_{\sim} = \{s' \in S \mid s \sim s'\}$. If $S$ is clear from the context, we may simply write $[s]^{S}_{\sim}$ as $[s]_{\sim}$.  It should be clear
that $[s]_\sim$~and~$[s']_\sim$ are equal if $s \sim s'$, and disjoint otherwise, so~$\sim$ induces a partition
$\partition{S}{\sim} = \{ [s]_\sim \mid s \in S \}$
of~$S$. We denote by $\sim_S^\star$ the transitive, reflexive, and symmetric closure of ${\sim}\cap(S\times S)$. In fact, we shall also use this notation if $\sim$ is an equivalence relation on a \emph{subset} of $S$. In this case, ${\sim_S^\star}={\approx_S^\star}$, where ${\approx}={\sim}\cup\{(w,w)\mid w\in S\}$.

\subsection{Strings, Languages, and Automata}
An alphabet is a finite nonempty set. Given an alphabet $\Xi$ we write $\Xi^*$ for the set of all strings over~$\Xi$, and~$\emptystr$ for the empty string. The lexicographic order on $\Xi^*$ is denoted by $\lex$.
A string language is a subset of $\Xi^*$.  Let $u, w \in \Xi^*$. The string $u$ is a \emph{prefix} of $w$ if there is a string $v \in \Xi^*$ such that $uv=w$. We denote the set of all prefixes of $w$ by $\prefixes{w}$. Similarly for a set of strings $S \subseteq \Xi^*$, $\prefixes{S} = \cup_{w \in S} \prefixes{w}$.

We denote a (partial) \emph{deterministic finite-state automaton} (DFA) is a tuple $M=(Q, \Xi, \delta, q_0, F)$, where $Q$ is a finite set of \emph{states}, $\Xi$ is an alphabet of \emph{input symbols}, the \emph{transition function} $\delta$ is a partial function $\delta\colon Q\times \Xi\to Q$, $q_0 \in Q$ is an \emph{initial state}, and $F \subseteq Q$ is a set of \emph{final states}. As usual, we denote by $\delta_M$ the partial function $\delta_M\colon \Xi^*\to Q$ such that $\delta_M(\emptystr)=q_0$ and $\delta_M(wa)=\delta(\delta_M(w),a)$ for all $w\in\Xi^*$ and $a\in\Xi$ such that the right-hand side $\delta(\delta_M(w),a)$ is defined. The language $L(M)$ \emph{recognised} by $M$ is defined in the standard way: $L(M)=\{w\in\Xi^*\mid\text{$\delta_M(w)$ is defined and in $F$}\}$. A string language is \emph{regular} if and only if it is recognised by some DFA. A DFA $N=(Q', \Xi, \delta', q_0, F)$ is a \emph{subautomaton} of $M$ if $Q' \subseteq Q$ and $\delta' \subseteq \delta$ (if $\delta$ is viewed as a subset of $Q \times \Xi \times Q$). 

Let $L \subseteq \Xi^*$ be a string language and $w \in \Xi^*$. To simplify notation, may use $L$ as a predicate and write $L(w)$ to denote the statement $w \in L$. The \emph{Nerode congruence} $\sim_L$ with respect to $L$ is the binary relation on $\Xi^*$ defined by $u \sim_L v$ if and only if $L(uw) = L(vw)$ for all $w \in \Xi^*$. It is well-known that $\sim_L$ is an equivalence relation, and as such partitions $\Xi^*$ into a set of congruence classes $\partition{\Xi^*}{\sim_L}$. We recall from the Nerode theorem that $\card{\partition{\Xi^*}{\sim_L}}$ is finite if and only if $L$ is regular. A string $u\in\Xi^*$ is \emph{live} (with respect to $L$ if $u=\emptystr$ or $uw\in L$ for some $w\in\Xi^*$. A congruence class in $\partition{\Xi^*}{\sim_L}$ is live if it contains a live string. If $L$ is regular, the canonical DFA recognizing $L$ is $M_L=(Q,\Xi,\delta,q_0,F)$ where $Q=\{B\in\partition{\Xi^*}{\sim_L}\mid\text{$B$ is live}\}$ and $q_0=\block{\emptystr}{\sim_L}$. (Note that we consider the partial canonical DFA $M_L$ that contains only the live congruence classes as states.)



\subsection{Terms and Algebras}
For a set $\types$ of types, a \emph{$\types$-typed alphabet} or simply \emph{typed alphabet} is a pair $(\Sigma,\typemapping_\Sigma)$ such that $\Sigma$ is a finite set of symbols and $\typemapping_\Sigma\colon\Sigma\to\types^+$ assigns to every $\sigma \in \Sigma$ a type signature $\typemapping_\Sigma(\sigma) \in \types^+$. If $\typemapping_\Sigma(\sigma)=\type_1\cdots\type_k\type$, we usually indicate this by writing $\sigma\colon\type_1\cdots\type_k\to\type$, leaving $\typemapping_\Sigma$ implicit. If $k=0$, we furthermore abbreviate $\sigma\colon\type_1\cdots\type_k\to\type$ by $\sigma\colon\type$, and if all symbols $\sigma\in\Sigma$ satisfy $\card{\typemapping_\Sigma(\sigma)}=1$, then we call $\Sigma$ a \emph{leaf alphabet}. 

Let $\Sigma$ be a $\types$-typed alphabet. The family $(\terms^\type)_{\type\in\types}$ of sets $\terms^\type$, called \emph{terms of type $\type$ over $\Sigma$}, is defined by simultaneous induction: $\terms^\type$ is the smallest set of formal expressions $f[t_ 1,\dots,t_k]$ such that $f[t_ 1,\dots,t_k]\in\terms^\type$ for all $f\colon\type_1\cdots\type_k\to\type$ in $\Sigma$ and $t_1\in\terms^{\type_1},\dots,t_k\in\terms^{\type_k}$.\footnote{We assume that the meta-symbols `$[$', `$]$', and the comma are not in $\Sigma$.}

Regular tree grammars are a well-known formalism for generating sets of terms. To recall them briefly here, such a grammar is a tuple $g=(\Sigma,N,R,S)$ where $\Sigma$ and $N$ are disjoint \type-typed alphabets for some set \types of types, where $N$ is a leaf alphabet, $R$ is a finite set of rules of the form $A\to t$ where $A\colon\type$ is in $N$ and $t\in\terms^\type$ for some $\type\in\types$, and $S\in N$. Derivations according to $g$ start with $S$ and replace, in each step, an occurrence of a leaf $A\in N$ in the current term $s$ by a term $t$ such that $(A\to t)\in R$. If $s'$ is the resulting term, this is written $s\to_gs'$. The language $L(g)$ generated by $g$ is the set of all term $s\in\terms^\type$ such that $S\to_g^* s$, where $\to_g^*$ denotes the transitive reflexive closure of $\to_g$. Note that $L(g)\subseteq\terms^\type$, where \type is the type of $S$, in other words, $S\colon\type$.

Let $\Sigma$ be a $\types$-typed alphabet. A $\Sigma$-algebra $\alg$ is a pair $((\dom_\type)_{\type \in \types},(f_\alg)_{f\in\Sigma})$ where $\dom_\type$ is a set for all $\type\in\types$ and $f_\alg$ is a function that maps $\dom_{\type_1} \times \cdots \times \dom_{\type_k}$ to $\dom_\type$ for every $f\colon \type_1 \cdots \type_k\to\type$ in $\Sigma$. A term $t\in\terms^\type$ can now recursively be evaluated with respect to $\alg$ as usual: if $t=f[t_1,\dots,t_k]$ then $\val_\alg(t)=f_\alg(\val_\alg(t_1),\dots,\val_\alg(t_k))$. Thus, $\val_\alg(t)\in \dom_\type$.

We extend the evaluation of terms to terms with variables in a standard manner. For this, let $X_\ell$ (where $\ell\in\nat$) be a \types-typed leaf alphabet of symbols $x_1,\dots,x_\ell$ called variables, where $X_\ell$ is disjoint with $\Sigma$. Now, consider a term $t\in\terms[X_\ell]^\type$, where $\typemapping_{X_\ell}(x_i)=\type_i$ for all $i\in[\ell]$. Then $\val_\alg^{X_\ell}(t)$ is the function $\varphi\colon \dom_{\type_1}\times\cdots\times \dom_{\type_\ell}\to \dom_\type$ defined recursively as follows, for all $a_1\in \dom_{\type_1},\dots,a_\ell\in \dom_{\type_\ell}$:
\begin{enumerate}
    \item If $t=x_i$ for some $i\in[\ell]$, then $\varphi(a_1,\dots,a_\ell)=a_i$.
    \item If $t=f[t_1,\dots,t_k]$, then $\varphi(a_1,\dots,a_\ell)=f(\val_\alg^{X_\ell}(t_1)(a_1,\dots,a_\ell),\dots,\val_\alg^{X_\ell}(t_k)(a_1,\dots,a_\ell)$.
\end{enumerate}
Note that this definition is consistent with the evaluation of terms without variables when $\ell=0$, i.e., $\val_\alg(t)=\val_\alg^\emptyset(t)$ for terms $t\in\terms^\type=\terms[\emptyset]^\type$. In the following, we will simply write $\val_\alg$ instead of $\val_\alg^{X_\ell}$ since $X$ will be understood from the context (or be sufficiently generic to be of lesser interest). In the following, we shall further generally assume that $X_\ell$ denotes an appropriate alphabet of \types-typed variables, if \types is understood from the context.


\section{Learning from Algebraic Expressions}
\label{sec:algebraic}

We now formalise an algebraic framework for learning and grounding.
Let $\Sigma$ be a \type-typed alphabet. A \emph{template $\Sigma$-algebra (over \types)} is defined just like a $\Sigma$-algebra over \types, with the exception that $f_\alg$ is undefined for some of the symbols $f\in\Sigma$. These $f\in\Sigma$ represent the functions to be learned. An \emph{instance} of $\alg'$ is a $\Sigma$-algebra such that $f_{\alg'}=f_\alg$ for all $f\in\Sigma$ for which $f_\alg$ is defined.

Consider a template $\Sigma$-algebra as above, and assume that a type $\etype\in\types$ is designated as the \emph{evaluation type}, for which an additional linear order $\le$ on $\dom_\etype$ is specified, as well as a summation operation $\bigoplus $ that is well-defined on every finite subset of $\dom_\etype$. Thus, every finite subset of $\dom_\etype$ has uniquely defined maxima and minima. An \emph{example} is a pair $(\varphi,O)$ consisting of a term $\varphi\in\terms[X_\ell]^\etype$ and a set $O\subseteq\bigcup_{\type\in\types}\dom_\type$. Suppose that $\typemapping_{X_\ell}(x_i)=\type_i$. Then the set $\groundings O{X_\ell}$ of all \emph{groundings of $O$} is the set of all injective mappings $g\colon X_\ell\to O$ such that $g(x_i)\in \dom_{\type_i}$ for all $i\in[\ell]$. Given an instance $\alg'$ of $\alg$, the \emph{value} of an example $S=(\varphi,O)$ is
\begin{equation}
    \val_{\alg'}(S)=\opt_{g\in\groundings O{X_l}}\val_{\alg'}(\varphi)(g(x_1),\dots,g(x_\ell)) \enspace, \label{eq:value}
\end{equation}
where $\opt\in\{\min,\max\}$ is a specified optimisation criterion.

Given a finite set $\samp$ of examples, the learning goal is to find an instance $\alg'$ of $\alg$ that optimises the total value
\begin{equation}\label{eq:optimisation goal}
    \bigoplus_{S\in\samp}\val_{\alg'}(S)
\end{equation}
of all examples. Since the desired outcome $\alg'$ depend on which algebras $\alg'$ are admissible solutions, these must be made explicit when formalising the learning task. In particular, for every $f\colon\type_1\cdots\type_k\to\type$ in $\Sigma$ such that $f_{\alg}$ is undefined, we specify a set $\mathcal F_f$ of candidate functions $g\colon\dom_{\type_1}\times\cdots\times\dom_{\type_k}\to\dom_\type$ for $f_{\alg'}$.

\section{Example Instantiations of the Framework}
\label{sec:applications}
To illustrate how the algebraic framework of Section~\ref{sec:algebraic} can be instantiated to capture various scenarios for algorithmic learning, we now describe three application domains. 

\subsection{Regular languages\label{se:regular}}
The first instance we consider is the grammatical inference of regular languages from finite data. If the data constitutes a \emph{characteristic set}~\citep{Higuera:1997}, then a finite automaton for the target language can be computed in polynomial time in the size of that set~\citep{OncinaGarcia:1992}.
Intuitively, a characteristic set contains strings that represent every state and every transition of the unique minimal DFA for the target language. We discuss a similar setting using our framework, in order to illustrate how traditional automata learning fits into it. 

Throughout the remainder of this subsection, let $\Xi$ be a finite, totally ordered, set of symbols, and $L \subseteq \Xi^*$ be a fixed but arbitrary regular language.
Let the set of types be $\Gamma = \{\alpha, \beta\}$, the \types-typed alphabet be $\Sigma = \{\accept \colon \alpha \to \beta,\, \equ \colon \alpha \times \alpha \to \beta,\, \neg \colon \beta \to \beta\}$, and the algebra be $\alg = ((\dom_\gamma)_{\gamma \in \Gamma}, (f_\alg)_{f \in \Sigma})$, where $\dom_\alpha = \Xi^*$ and $\dom_\beta = \{\mathrm{false},\mathrm{true}\}$. The operation $\neg_\alg$ is logic negation. The operations $\accept_\alg$ and $ \equ_\alg$ are both undefined: The characteristic function $\accept_\alg\colon \Xi^* \to \{\mathrm{false},\mathrm{true}\}$ for the language $L$ is our learning target, and $\equ_\alg$ will be instantiated to model the Nerode congruence $\sim_L$ of $L$. Thus, admissible instantiations are algebras $\alg_L$ in which $\accept_{\alg'}$ is the characteristic function of a regular language $L$ and $\equ_{\alg'}$ the Nerode congruence of $L$.
 
Note that, in this instance of our framework, an example is a pair $(\varphi, O)$ where $\varphi$ can be assumed to be taken from $\{\accept[x], \neg[\accept[x]], \equ[x,y], \neg [\equ[x,y]]\}$ for variables $x,y$ and a set $O$ of strings. For the two positive terms $\accept[x]$ and $\equ[x,y]$, we furthermore restrict ourselves to examples such that the number of strings in $O$ is equal to the number of variables in $\varphi$, i.e., one or two.
We let $\textrm{false} < \textrm{true}$. In Equation~\eqref{eq:value}, we let $\opt=\min$. (Note that this choice is important only for the examples using negation, because the positive examples have as many objects as variables, and $\equ_{\alg'}$ is required to be symmetric.)
Finally, for Equation~\eqref{eq:optimisation goal}, we take $\bigoplus$ to be $\bigwedge$, that is, logic conjunction.  Thus, the optimisation goal is to find a congruence relation that respects the partial information provided through the set of examples. 

The mapping $\mathit{pred} \colon \partition{\Xi^*}{\sim_L} \to \powerset{ \partition{\Xi^*}{\sim_L} \times \Xi}$ is defined as $\mathit{pred}(B) = \{ (D,\xi) \mid D\xi \subseteq B\}$ for every $B \in \partition{\Xi^*}{\sim_L}$. Thus, intuitively, in the minimal DFA for $L$, in which the live congruence classes of $\sim_L$ are the states, $\mathit{pred}(B)$ is the set of pairs of predecessor states of $B$, together with the transition symbols that take these to $B$. 
The congruence class $B$ is a \emph{convergence} if $|B\cap\{\emptystr\}| + \card{\pred{B}} > 1$.
Let $\samp$ be a set of examples and denote by $\strings{\samp}$ the set of all strings that occur in examples of the forms $( \accept[x], \{w\})$ and $( \equ[x,y], \{u, w\})$.  The set $\samp$ is \emph{sufficient} for $L$ if:
\begin{enumerate}
\item \label{conds:faithful} $\samp$ is \emph{faithful} (with respect to $L$), meaning that, for every example $(\varphi,O)\in\samp$, the set $O$ contains only live strings and there is a grounding $g\in\groundings O{X_l}$ such that $\val_{\alg_L}(g(x_1),\dots,g(x_\ell))=\mathrm{true}$ (where $\ell$ is such that $\varphi\in\terms[X_\ell]^\beta$).
\item\label{conds:congruence} For every live $B \in \partition{\Xi^*}{\sim_L}$,
\begin{enumerate}
    \item \label{cond:accepting} If $B\subseteq L$, then $( \accept[x], \{w\}) \in \samp$ for some $w\in B$.
    \item \label{cond:intersection} If $B$ is a convergence, then for every $(D,\xi) \in \pred{B}$,
    \[
    D\xi \cap \strings{\samp} \, =\,  D\xi \cap \prefixes{\strings{\samp}} \, \not = \,  \emptyset \enspace .
 \]

\end{enumerate}
\item \label{cond:smaller} For every $w \in \strings{\samp}$ that is not minimal (with respect to $\lex$) in $[w]^{\strings{\samp}}_{\sim_L}$ there is a $u$ in $[w]^{\strings{\samp}}_{\sim_L}$ such that $u\lex w$ and $( \equ[x,y], \{u, w\}) \in \samp$.
\item \label{cond:distinguish} For every pair $B, B'\in\partition{\Xi^*}{\sim_L}$ of live congruence classes, $B\neq B'$ if and only if $\samp$ contains an example $(\neg \equ[x,y], O)$ such that $|O\cap B|=1=|O\cap B'|$.
\end{enumerate}

For the regular language $L = (ab)^*$, the set of examples in Figure~\ref{tab:regular} is sufficient and thus, as we shall see below, allows us to infer $L$. Note, however, that $L$ is not the only regular language which solves the optimisation problem given by~\eqref{eq:optimisation goal}. Another one would, for instance, be the language~$(a^+b)^*$ obtained from the DFA in Figure~\ref{tab:regular} by adding a transition which is not required by the examples but not forbidden either.
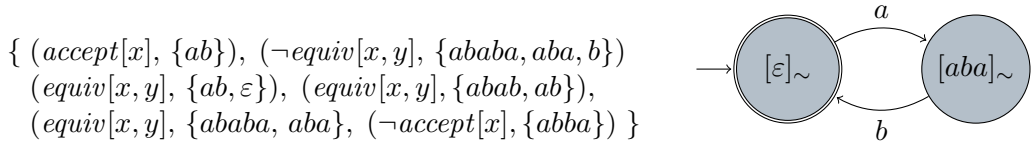
\begin{figure}[t]
\begin{minipage}{.5\textwidth}
\centering
\[
\begin{array}{c@{\;}l@{}c}
     \{ & (\accept[x],\, \{ab\}), \; (\neg \equ[x,y],\, \{ababa, aba, b\})\\
     & ( \equ[x,y],\, \{ab,\emptystr\}), \; ( \equ[x,y], \{abab,ab\}),\\ 
     & ( \equ[x,y],\, \{ababa,\, aba\}, \; (\neg \accept[x], \{abba\})   & \} \\
\end{array}
\]
\end{minipage}%
\begin{minipage}{.5\textwidth}
\centering
\begin{tikzpicture}[shorten >=1pt,node distance=2.5cm,on grid,auto]
  \tikzstyle{every state}=[fill={UmUBlue!35},minimum size=1.4cm]
  \node[state,initial,accepting,initial text={}]  (s_0)                 {$[\varepsilon]_\sim$};
  \node[state]                    (s_1) [right of=s_0]  {$[aba]_\sim$};

  \path[->]
  (s_0) edge [bend left] node {$a$}  (s_1)
  (s_1) edge   [bend left]  node {$b$}  (s_0);
\end{tikzpicture}
\end{minipage}
    \caption{A sufficient set of examples $\samp$ to allow the inference of the regular language $(ab)^*$, and the automaton derived from $\samp$. We note that the last three examples could be dropped, and the set would still be sufficient. }
    \label{tab:regular}
\end{figure}

In the following, let $L$ be a regular language, $\samp$ a sufficient example for $L$, and $\sim_\samp$ the smallest equivalence relation on $\strings\samp$ such that $w\sim_\samp w'$ for all $w,w'\in\strings\samp$ such that $\samp$ contains the example $(\equ(x,y),\{w,w'\})$.

\begin{lemma}
\label{lm:liven}
If $B \in \partition{\Xi^*}{\sim_L}$ is live then for every $(D,\xi) \in \pred{B}$,  $D\xi \cap \prefixes{\samp} \not = \emptyset$. 
\end{lemma}
\begin{proof}
 If $B$ is a convergence, the statement of Lemma~\ref{lm:liven} is immediately fulfilled by Condition~\ref{cond:intersection}. Assume then that this is not the case, and choose some arbitrary string $v\in B$. We cannot have $B=\block{\emptystr}{\sim_L}$ because then the assumption that $B$ is not a convergence implies that $\pred B=\emptyset$. Since $B$ is live, there is a string $w$ such that $[vw]_{\sim_L}$ is a convergence or a subset of $L$. Let $w$ be such a string of minimal length (this can be $0$). Because of the minimality of $w$ there must be a string $u \in B$ such that $uw \in \strings{\samp}$. Since $\emptystr\notin B$, we have $u=u_0\xi$ for some $u_0\in\prefixes\samp$ and $\xi \in \Xi$. Then $\pred{B} = \{([u_0]_{\sim_L}, \xi)\}$, which completes the proof.
\end{proof}

\begin{definition}
\label{def:simrc}
Let $S\subseteq\Xi^*$, and let $\sim$ be an equivalence relation on $S$. The \emph{right completion} of $\sim$ is the smallest equivalence relation ${\simrc}\supseteq{\sim}$ on $\prefixes S$ such that, for all $w\xi,w'\xi\in\prefixes S$ (where $w,w'\in\Xi^*$ and $\xi\in\Xi$), if $w\simrc w'$ then $w\xi\simrc w'\xi$.
\end{definition}

Note that, for finite $S$, the set $E$ of equivalence classes of $\simrc$ (and thus $\simrc$ itself) can be computed from $\sim$ in a standard manner: initially, $E=\partition{\prefixes S}\sim$. Now, while there exists distinct $C, C'\in E$ such that there are $w\xi\in C$ and $w'\xi \in C'$ such that $w \simrc w'$, we let $E \leftarrow (E\setminus\{C,C'\})\cup \{C\cup C'\}$.

We have the following lemma:

\begin{lemma}
\label{lm:simrc}
${\simrc_\samp} = {\sim_L}\cap(\prefixes\samp\times\prefixes\samp)$.
\end{lemma}
\begin{proof}
Since right completion is monotone and ${\sim_\samp}\subseteq{\sim_L}$, we have ${\simrc_\samp} \subseteq {\simrc_L}={\sim_L}$, where the equality holds because $\sim_L$ is a right congruence. Here, ${\simrc_\samp} \subseteq \prefixes\samp\times\prefixes\samp$, and thus ${\simrc_\samp} = {\simrc_\samp} \cap(\prefixes\samp\times\prefixes\samp)\subseteq {\sim_L}\cap(\prefixes\samp\times\prefixes\samp)$.

For the other inclusion, assume that $w, w' \in \prefixes{\samp}$ are such that $w \sim_L w'$. Due to Condition~\ref{cond:intersection}, $w$ and $w'$ must either both be in $\strings{\samp}$ or both be in $\prefixes{\samp} \setminus \strings{\samp}$.  If they are in $\strings{\samp}$, then by Condition~\ref{cond:smaller}, $w \sim_\samp^\star w'$.  If they are in $\prefixes{\samp} \setminus \strings{\samp}$, then they must be of the form $u\xi, u'\xi$ for some $u,u' \in \prefixes{\samp}$ with $u \sim_L u'$, or else $\block w{\sim_L}$ would constitute a convergence and thus $w,w'\in\strings{\samp}$. Assume that $w, w'$ are of minimal length with this property. Then $u\simrc_\samp u'$, but this cannot be because $w,w'\in\prefixes\samp$ and thus $w\simrc_\samp w'$ by Definition~\ref{def:simrc}.
\end{proof}

\begin{definition}[DFA construction]
\label{def:aut}
Let $S \subseteq \Xi^*$, let $\sim$ be an equivalence relation on~$S$, and let $F \subseteq S$. The DFA $\aut{S,\sim,F}$ is given by $(\Xi, Q, \delta, q_0, Q_F)$ where
\begin{itemize}
\item $Q = \partition{\prefixes S}{\sim}$,
\item $\delta$ is the partial transition function such that $\delta(B,\xi)$ is defined and equal to $B'\in Q$ if $w\xi\in B'$ for some $w\in B$,
\item $q_0 = [\emptystr]^{\prefixes S}_\sim$, and
\item $Q_F = \{B \in Q \mid B \cap F \neq \emptyset\}$.
\end{itemize}
\end{definition}

With the preceding definition, we can now use $\samp$ to build a DFA that recognises $L$:

\begin{lemma}
\label{lm:aut}
Let $F_\samp=\{w \in \strings{\samp} \mid (\accept[x],\{w\}) \in \samp\}$. Then the DFA $M_\samp=\aut{\prefixes{\samp},\simrc_\samp,F_\samp}$ is isomorphic to $M_L$.
\end{lemma}

\begin{proof}

Let $M_L=(\Xi,Q',\delta',q_0',F')$ and $M_\samp=(\Xi,Q,\delta,q_0,F)$. By Lemma~\ref{lm:simrc}, ${\simrc_{\samp}}\subseteq{\sim_L}$, so every state $q\in Q$ is a subset of some (by Lemma~\ref{lm:liven} live) congruence class in $\partition{\Xi^*}{\sim_L}$, i.e., of a state $q'\in Q'$. Also by Lemma~\ref{lm:simrc}, $w\not\simrc_\samp w'$ implies $w\not\sim_Lw'$. Hence, the mapping $\pi\colon Q\to Q'$ such that $q\subseteq\pi(q)$ for all $q\in Q$ is well defined and injective.

Lemma~\ref{lm:liven} together with the constructions of $\delta$ and $\delta'$ (via Definition~\ref{def:aut} and the Myhill-Nerode construction, respectively), yield the following for all $q\in Q$ and $\xi\in\Xi$: $\delta(q,\xi)$ is defined if and only if $\delta'(\pi(q),\xi)$ is defined, and in this case it holds that $\delta'(\pi(q),\xi)=\pi(\delta(q,\xi))$. 
Moreover, $\pi(q_0)=q_0'$ because $\emptystr\in q_0\cap q_0'$. Finally, a state $q\in Q$ is in $F$ if and only if $q\cap F_\samp\neq\emptyset$, which by the definition of $F_\samp$ and Condition~\ref{cond:accepting} is the case if and only if $\pi(q)\cap L\neq\emptyset$. Again by the fact that $M_L$ results from the Myhill-Nerode construction, the latter is equivalent to $\pi(q)\in F'$. Hence, altogether, $M_\samp$ and $M_L$ are isomorphic.
\end{proof}

Note that the construction of $M_\samp$ does not depend on the existence of distinguishing examples according to Condition~\ref{cond:distinguish}. While this means that $M_\samp$ would stay unchanged even if those examples would be absent, the theorem we prove now would fail. 

\begin{theorem}\label{th:minimal}
Let $L$ be a regular language and $M_L=(\Xi,Q,\delta,q_0,F)$. Then
$\alg_L$ maximizes~\eqref{eq:optimisation goal}. Furthermore, it is minimal in the following sense: if $M=(\Xi,Q',\delta',q_0',F')$ is another DFA such that $\alg_{L(M)}$ maximizes~\eqref{eq:optimisation goal}, then there is an injective mapping $\pi\colon Q\to Q'$ such that $\pi(q_0)=q_0'$, $\delta'(\pi(q),\xi)=\pi(\delta(q,\xi))$ for all $(q,\xi)\in\domain(\delta)$, and $\pi(F)\subseteq F'$. In particular, $L\subseteq L(M')$.
\end{theorem}

\begin{proof}
It is straightforward to verify that the instantiation $\alg_L$ ensures that all examples attain the value $\mathrm{true}$. Hence, $\alg_L$ maximizes~\eqref{eq:optimisation goal}. 

For the rest of the proof, let $M_\samp=\aut{\prefixes\samp,\simrc_\samp,F_\samp}$. By Lemma~\ref{lm:aut} we prove the statement of the theorem for $M_\samp=(\Xi,Q,\delta,q_0,F)$ instead of $M_L$, i.e., the states are subsets of $\prefixes\samp$ instead of congruence classes of $\partition{\Xi^*}{\sim_L}$. Let $M$ be another DFA, as in the statement of the theorem. Instead of $M$, we shall for the time being consider the canonical \emph{total} DFA $M_0=(\Xi,Q_0',\delta_0',q_0',F')$ recognising $L(M)$. Thus, we can assume that $M$ is obtained from $M_0$ by removing the sink state if there is one, and $M=M_0$ otherwise.

Every state $q\in Q$ is a subset of a state in $Q_0'$ because
\begin{enumerate*}[label=(\alph*)]
\item for all $w,w'\in \prefixes\samp$ the implications $w\sim_\samp w'\Rightarrow (\equ(x,y),\{w,w'\})\in\samp\Rightarrow w\sim_{L(M_0)}w'$ hold (the latter since $M$ maximises \eqref{eq:optimisation goal}),
\item $\sim_{L(M_0)}$ is a right congruence and
\item $\simrc_\samp$ is the smallest right equivalence on $\prefixes\samp$ that contains $\sim_\samp$.
\end{enumerate*}

Let thus $\pi\colon Q\to Q_0'$ be the mapping such that $q\subseteq\pi(q)$ for all $q\in Q$. Then $\pi$ is injective: if there would be two distinct states $q,q'\in Q$ such that $q'\subseteq\pi(q)$ then $M_0$ (and thus $M$) would fail to satisfy one of the examples $(\neg\equ(x,y),O)$ contained in $\samp$ owing to Condition~\ref{cond:distinguish}, and would thus not maximise \eqref{eq:optimisation goal}. Furthermore, by Condition~\ref{cond:accepting} we have $\pi(q)\in F_\samp$ for all $q\in F$, and since $\emptystr\in q_0\cap q_0'$, we also have $\pi(q_0)=q_0'$.

Now recall that $\delta_0'$ is the transition function defined by the Myhill-Nerode construction. Thus, for all $q\in Q_0$, $w\in\Xi^*$, and $\xi\in\Xi$ it holds that $\delta_0'(q,a)=q'$ if $wa\in q'$. Since this is just how $\delta$ is defined, it follows that $\delta_0(\pi(q),\xi)=\pi(\delta(q,\xi))$ for all $(q,\xi)\in\domain(\delta)$. In particular, all states of the form $\pi(q)$, $q\in Q$, are live in $M_0$ because they are so in $M_\samp$. Consequently, $\pi$ actually maps $Q$ to $Q'$, which completes the proof.
\end{proof}

\subsection{Picture Languages}
The second example instantiation of the algebraic framework allows us to infer picture languages of the type described by~\cite{Drewes:06}. 
These language are generated by grammatical systems whose rules transform and combine images to more complex ones: A \emph{collage grammar} is a pair $G=(g,\calg)$ consisting of
\begin{enumerate*}[label=(\arabic*)]
\item a regular tree grammar $g$ such that $L(g)\subseteq\terms^p$ for some type $p$ and a $\{p\}$-typed alphabet $\Sigma$, and
\item a so-called collage algebra $\calg=(\collages,(F_{\calg})_{F\in\Sigma})$.
\end{enumerate*}
The components of such a collage algebra \calg are given as follows:
\begin{itemize}
\item \collages is the set of all compact subsets of $\real^2$, here referred to as \emph{pictures}.\footnote{A subset of $\real^2$ is compact if it is bounded (a subset of a disc) and contains all limits of Cauchy sequences in it.}
\item The interpretation of every $F\colon p$ in $\Sigma$ is an element $F_{\calg}$ of \collages.
\item The interpretation of every $F\colon p^k\to p$ with $k>1$, is given by $k$ affine transformation $f_1,\dots,f_k$ of $\real^2$, as follows: $F_{\calg}(C_1,\dots,C_k)=\bigcup_{i\in[k]}f_i(C_i)$. Thus, $F$ transforms its $i$th argument by $f_i$ and returns the union of the resulting $k$ pictures. We also denote such a function $F_{\calg}$ by $\cop{f_1\cdots f_k}$.
\end{itemize}
The language generated by $G$ is $L(G)=\{\val_{\calg}(t)\mid t\in L(g)\}$.

An example of a collage grammar $G=(g,\calg)$ is depicted in Figure~\ref{fi:collage grammar}.
\begin{figure}
\newlength{\picwidth}
\setlength{\picwidth}{1.5cm}
\newcommand{\rul}[3][]{%
  \def\tempa{#1}%
  \def\tempb{}%
  \ifx\tempa\tempb%
    \raisebox{-.5\picwidth + .5ex}{\includegraphics[width=\picwidth]{treebag/grid/chair/ps/#2-eps-converted-to.pdf}}\ \to \ \raisebox{-.5\picwidth + .5ex}{\includegraphics[width=\picwidth]{treebag/grid/chair/ps/#3-eps-converted-to.pdf}}%
  \else%
    \raisebox{-.5\picwidth + .5ex}{\includegraphics[width=\picwidth]{treebag/grid/chair/ps/#2-eps-converted-to.pdf}}\ \to \ \raisebox{-.5\picwidth + .5ex}{\includegraphics[width=\picwidth]{treebag/grid/chair/ps/#3-eps-converted-to.pdf}}%
    \quad\rule[-.5\picwidth]{.5pt}{1.1\picwidth}\quad\raisebox{-.5\picwidth + .5ex}{\includegraphics[width=\picwidth]{treebag/grid/chair/ps/#1-eps-converted-to.pdf}}%
  \fi%
}
\hspace*{\fill}$
\begin{array}{@{}c@{}}
\rul S{rhsS}\\\\[0pt]
\rul[rhsA2]A{rhsA1}\\\\[0pt]
\rul[rhsB2]B{rhsB1}
\end{array}
$
\hfill
$\begin{array}{@{}c@{}}\includegraphics[width=6cm]{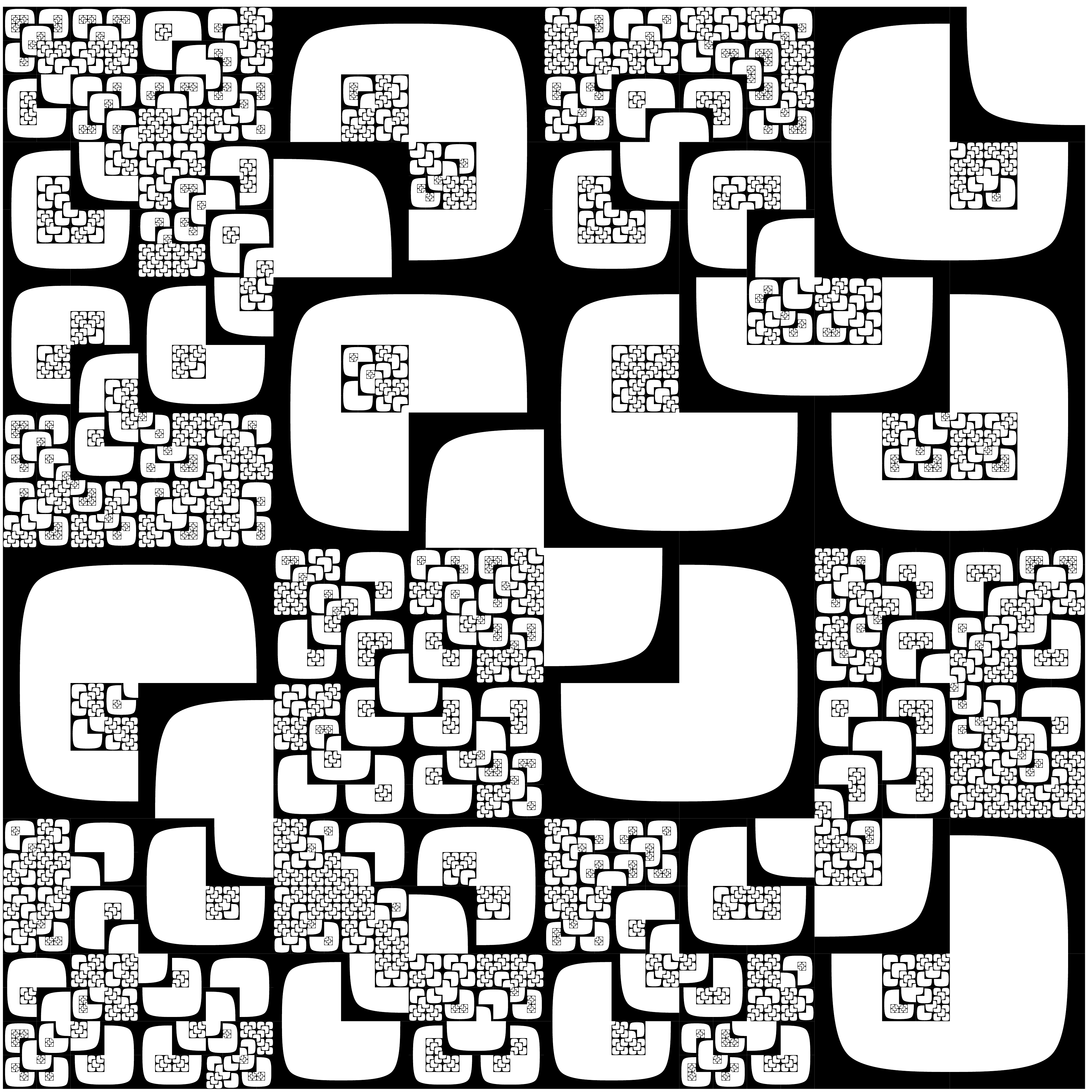}\end{array}$
\hspace*{\fill}

\caption{Graphical illustration of a collage grammar (left) and one of the generated pictures (right). The rules of the regular tree grammar are $S\to F[A,A,A,A]$ (top), $A\to F[A,A,A,A] \mid F[B,B,B,B]$ (middle), and $B \to G[C, S] \mid C$ (bottom), separating alternative right-hand sides for the same left-hand side by `$\mid$'. To obtain the graphical illustration of the rules and, at the same time, indicate the transformations constituting $F$ and $G$ (and the picture $C$), the algebra was extended to interpret $S$, $A$, and $B$ as squares with an inscribed nonterminal name and \textsf L-shape making orientation unambiguous. (Colours just serve better visual distinction.) The initial nonterminal is $S$.\label{fi:collage grammar}}
\end{figure}%
  Suppose we want to learn $G$ from examples. We obviously have to learn both $g$ and the interpretation of the symbols appearing in the terms in $L(g)$, i.e., we have to learn~\calg. Learning $g$ is an interesting task that has been considered in many research articles. However, since this problem is similar to Section~\ref{se:regular}, we do not consider it here, focusing instead on the second component. We discuss one way in which our framework can be used to formalise and study this learning task.

  Consider the operation $F_\calg=\cop{f_1,f_2,f_3,f_4}$ in our example. To learn it, we need to learn each of the affine transformations $f_i$ defining it. By the definition of affine transformations, $f_i(x)=M_ix+b_i$ where $M_i$ is a $2\times2$ matrix with real coefficients and $b\in\real^2$. This is just the type of mapping a single linear layer of a neural network $\mathcal N$ can implement. Hence, we propose to use our framework to study how this can be done. For the learning, we may assume that all operations except $F_\alg$ are known, i.e., we consider the template algebra $\calg'$ which is equal to $\calg$ with the exception that $F_{\calg'}$ is undefined. We can use the regular tree grammar $g$ to generate examples from which the network is supposed to learn $F_\alg$. Intuitively, each example is a pair consisting of a term and its evaluation according to $\calg$. However, this does not yet fulfil the requirements of our general definition. What is missing is a way to measure how good a job $\mathcal N$ does. In other words, we need a distance measure on $\collages$. For this, we may for example use the well-known Hausdorff measure~\citep{Hausdorff:1918}, the more recent Earth Mover's distance by~\cite{RubnerEtAl:2000}, or, if we restrict ourselves to pictures consisting of finite unions of simple geometric shapes, the area of the symmetric difference between pictures. Let $\delta$ denote such a distance measure, and let $\calg_\delta$ be $\calg$ extended by this operation, which is of type $\delta\colon\collages\collages\to\real$. (For simplicity, we blur the distinction between the symbol $\delta$ and its interpretation.) Now, to stay within our framework, we can use examples of the form $(\delta[x_1,t],\val_{\calg}(t))$, where $t\in L(g)$. In fact, as long as we are not concerned with learning $g$ as well, the requirement that $t\in L(g)$ is superfluous. Furthermore, we can enrich $\calg'$ with further shapes than only $C$, making it possible to use examples that simplify learning.

  Thus, to experiment with how well $\mathcal N$ can learn $F_\calg$, we may generate a corpus of such examples and train $\mathcal N$, using as a loss function the average of the losses of individual examples. Let $\calg_{\mathcal N}$ denote the algebra obtained from $\calg'$ by interpreting $F$ according to $\mathcal N$. The loss of a single example $(\delta[x_1,t],C)$ is then $\delta(C,\val_{\calg_{\mathcal N}}(t))$. An example of a set of four examples using triangles and squares as basic pictures is shown in Figure~\ref{fi:collage examples} at the top.%
\begin{figure}[t!]
\scriptsize
\newcommand{\scalefactor}{.05}%
\newcommand{\sq}{\mathit{sq}}%
\newcommand{\tri}{\mathit{tri}}%
\resizebox{\linewidth}{!}{%
$\begin{array}{@{}c@{\qquad}c@{\qquad}c@{\qquad}c@{}}
F[\tri,\sq,\tri,\sq] & F[\tri,\tri,\sq,\sq] & F[\sq,\tri,\sq,\tri] & F[\begin{array}[t]{@{}l@{}}F[\tri,\sq,\tri,\sq],\\\tri,\tri,\sq]\end{array}\\
\includegraphics[scale=\scalefactor]{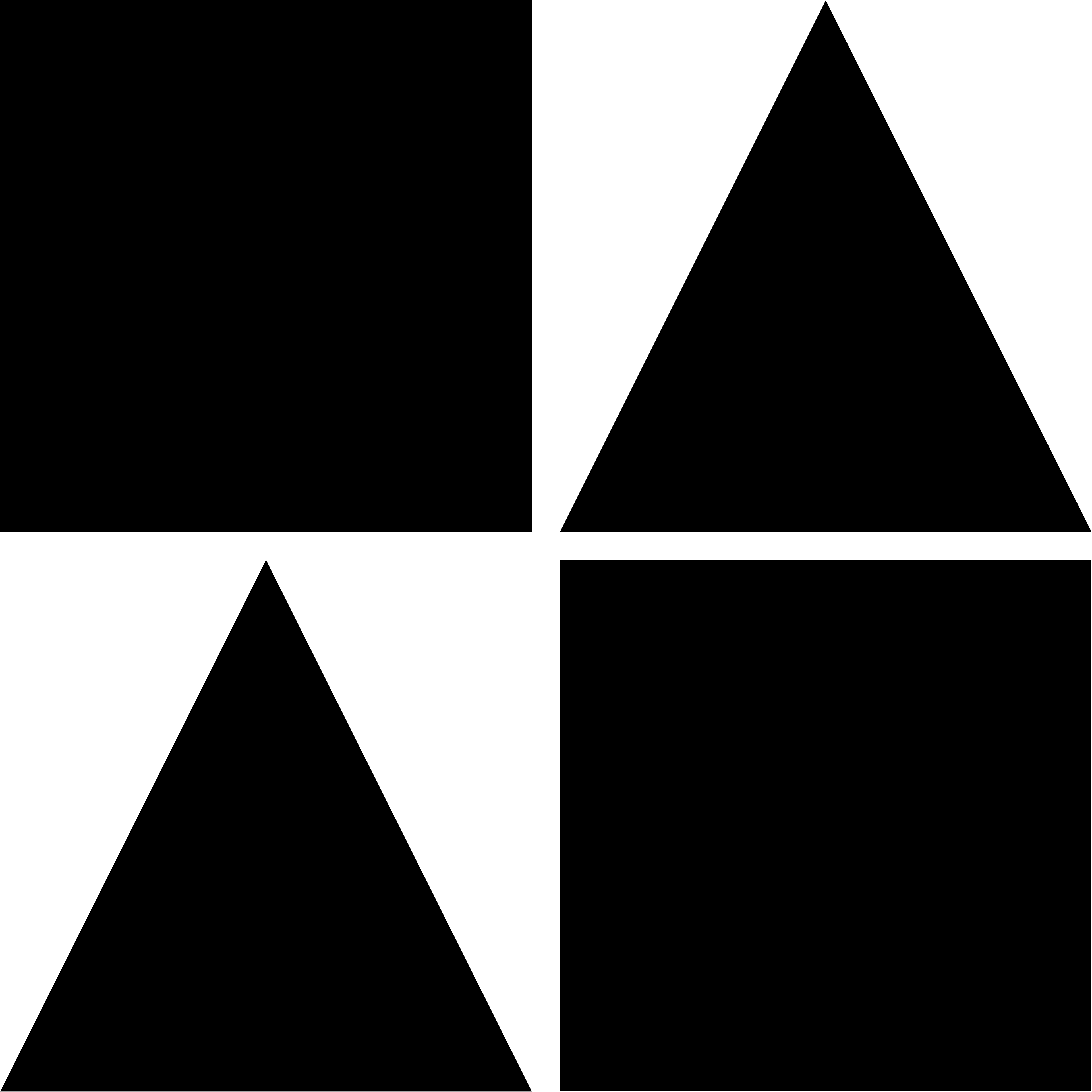} & \includegraphics[scale=\scalefactor]{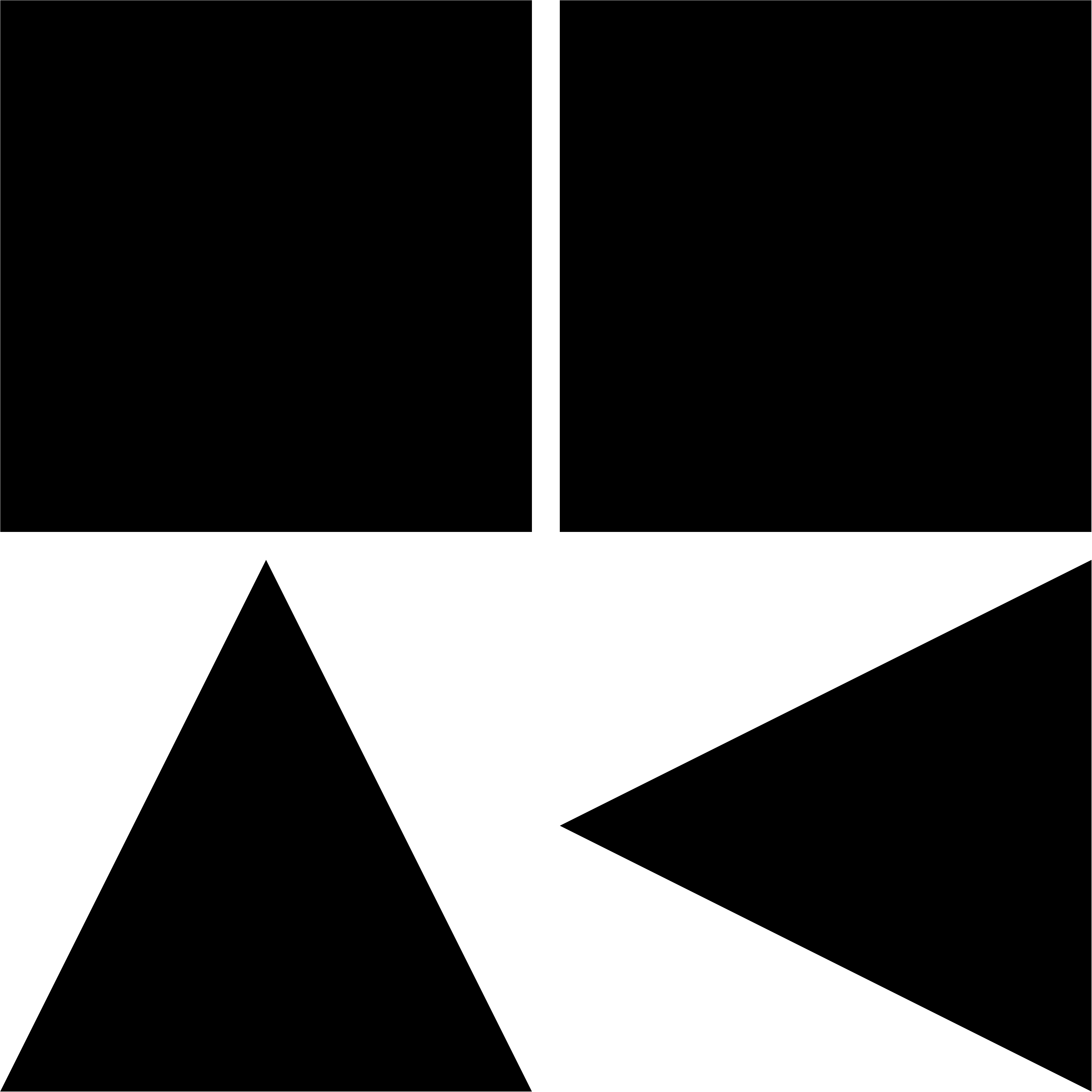} & \includegraphics[scale=\scalefactor]{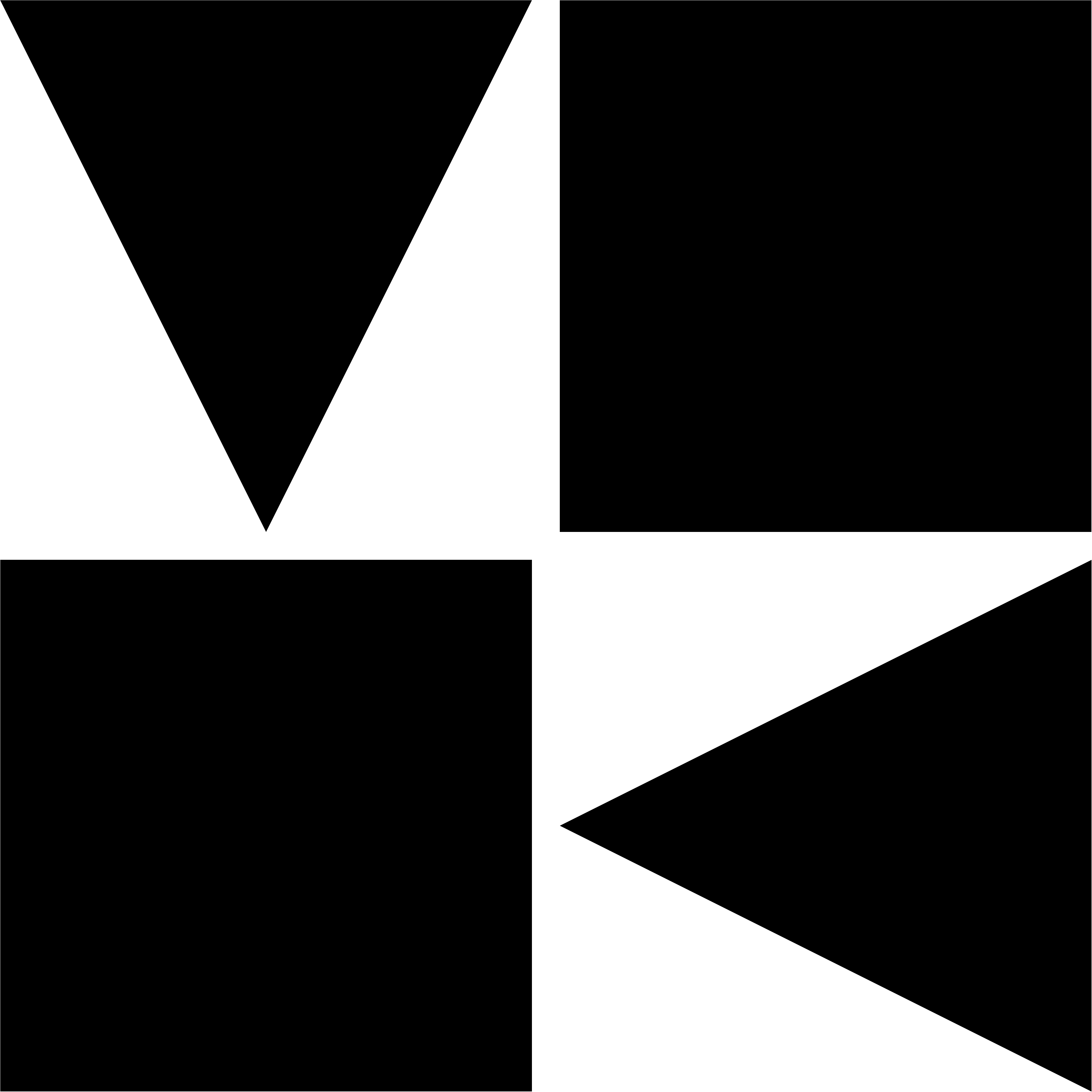} & \includegraphics[scale=\scalefactor]{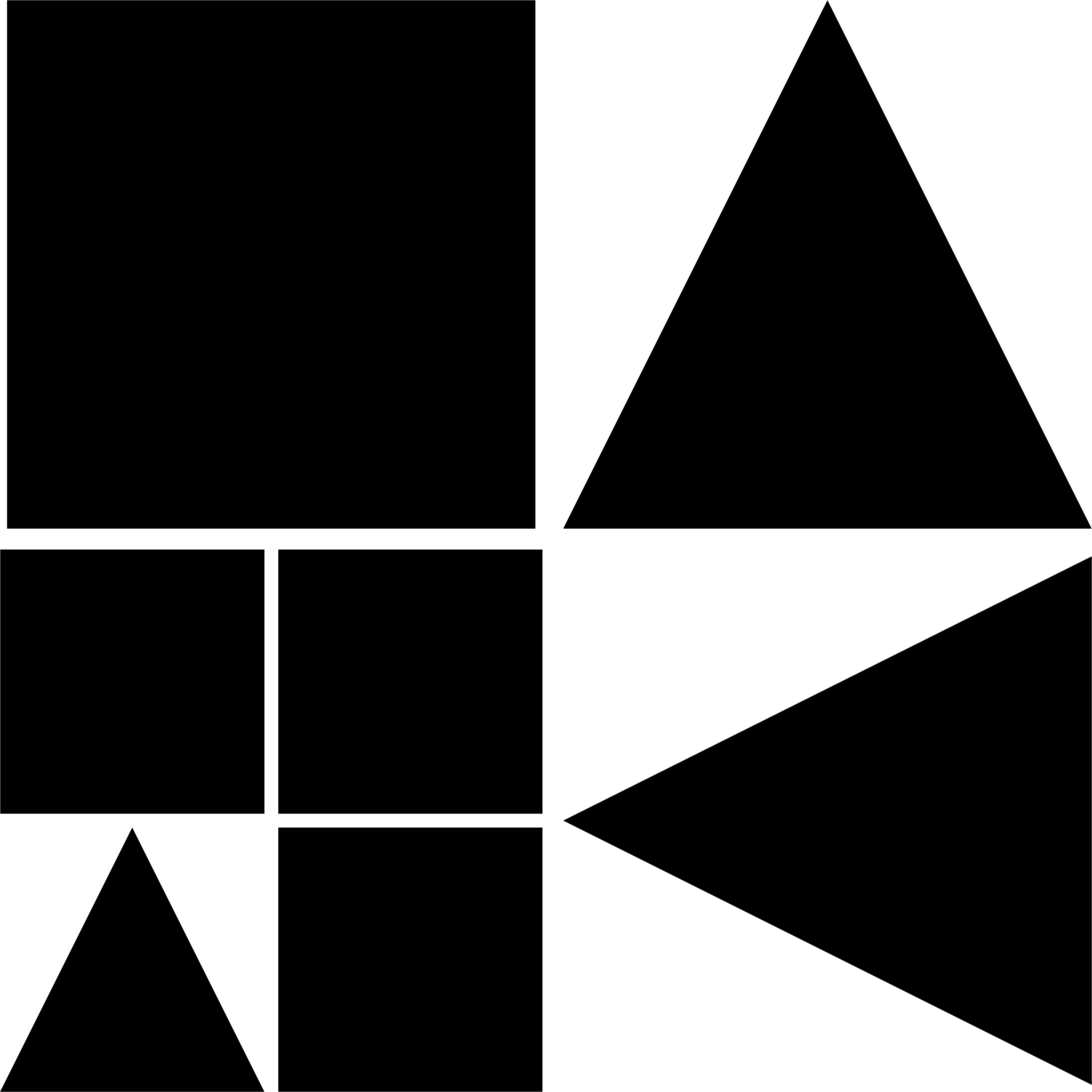}\\\\[0pt]
\includegraphics[scale=\scalefactor]{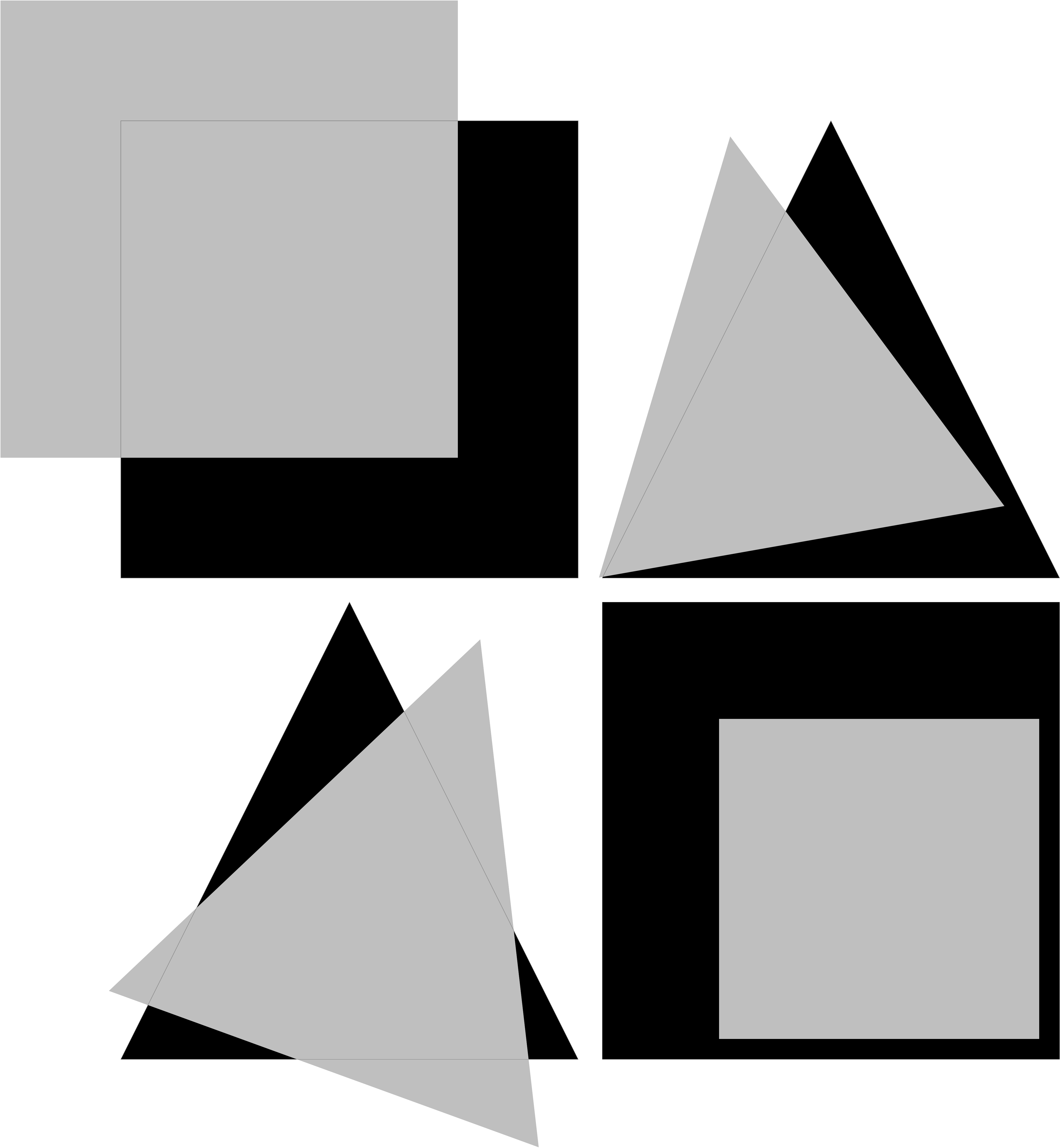} & \includegraphics[scale=\scalefactor]{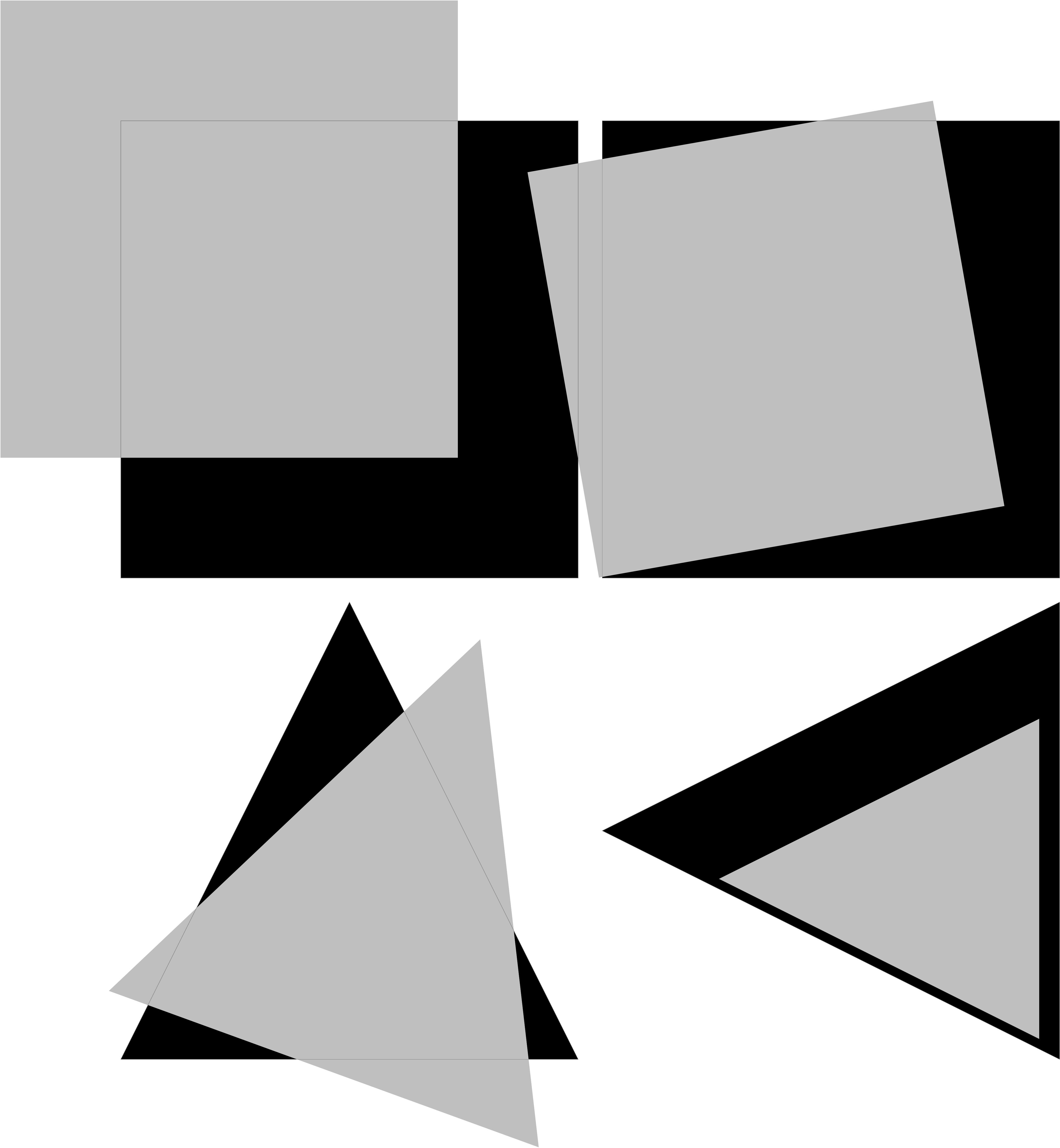} & \includegraphics[scale=\scalefactor]{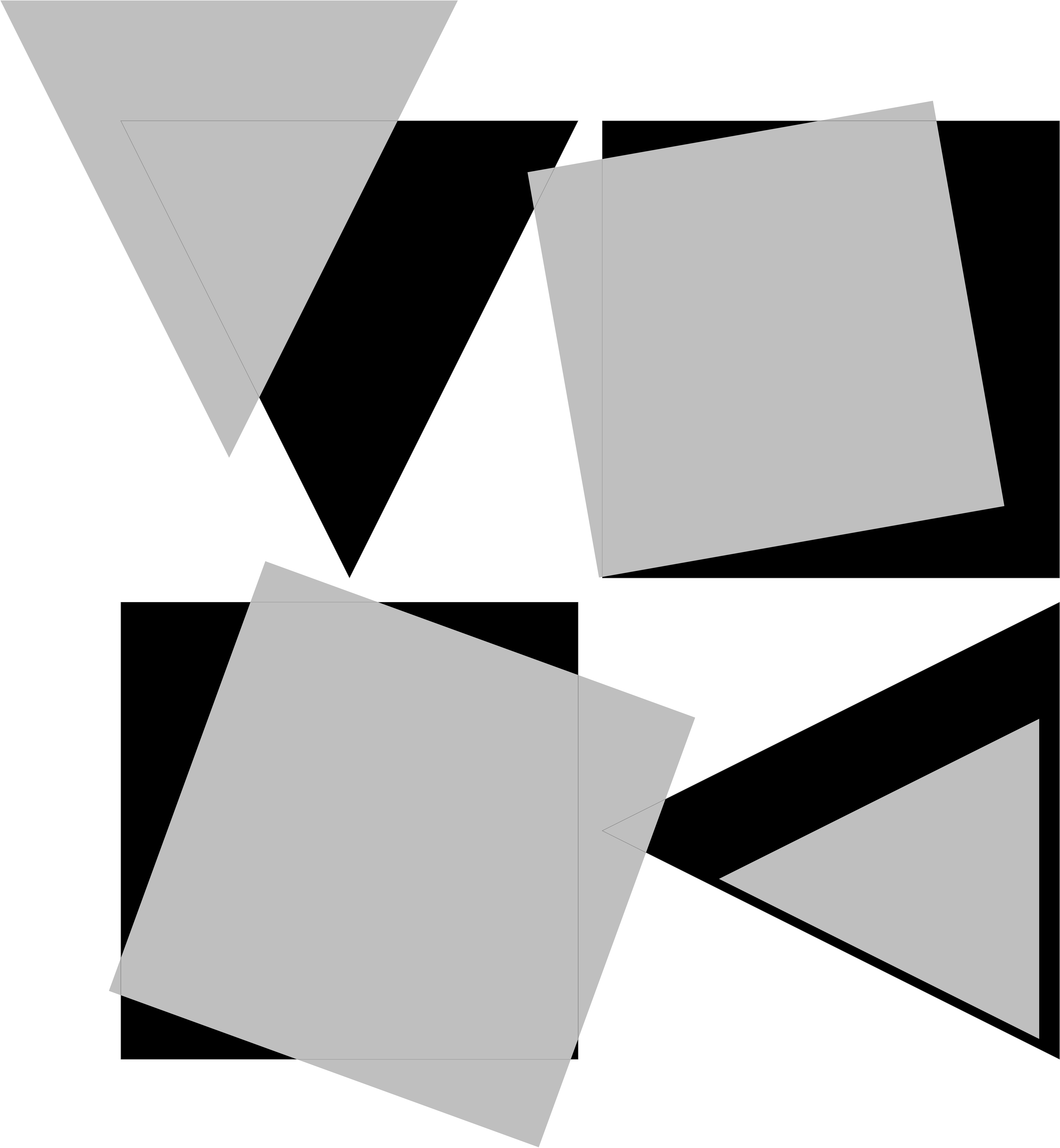} & \includegraphics[scale=\scalefactor]{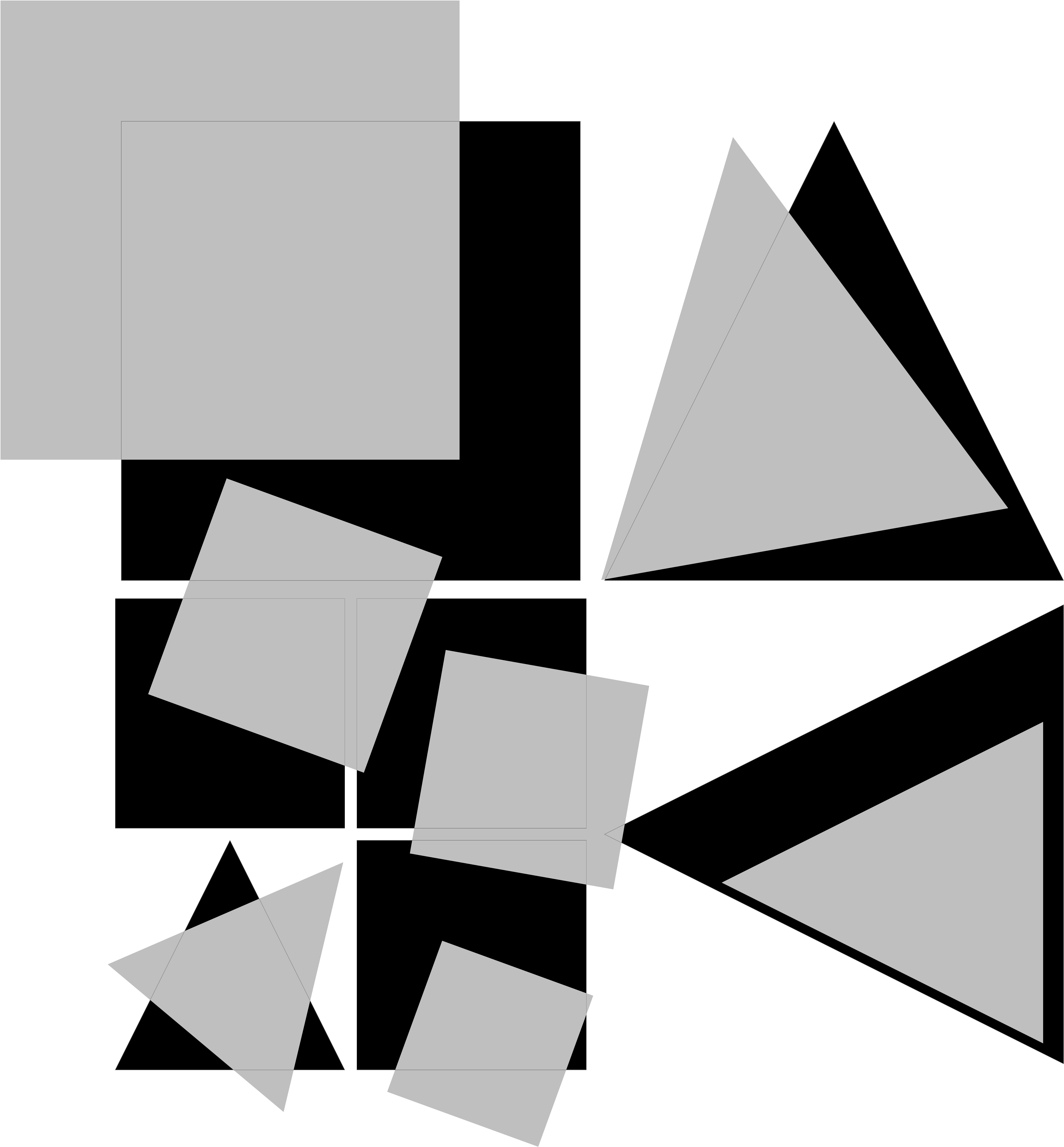}\\\\[0pt]
\includegraphics[scale=\scalefactor]{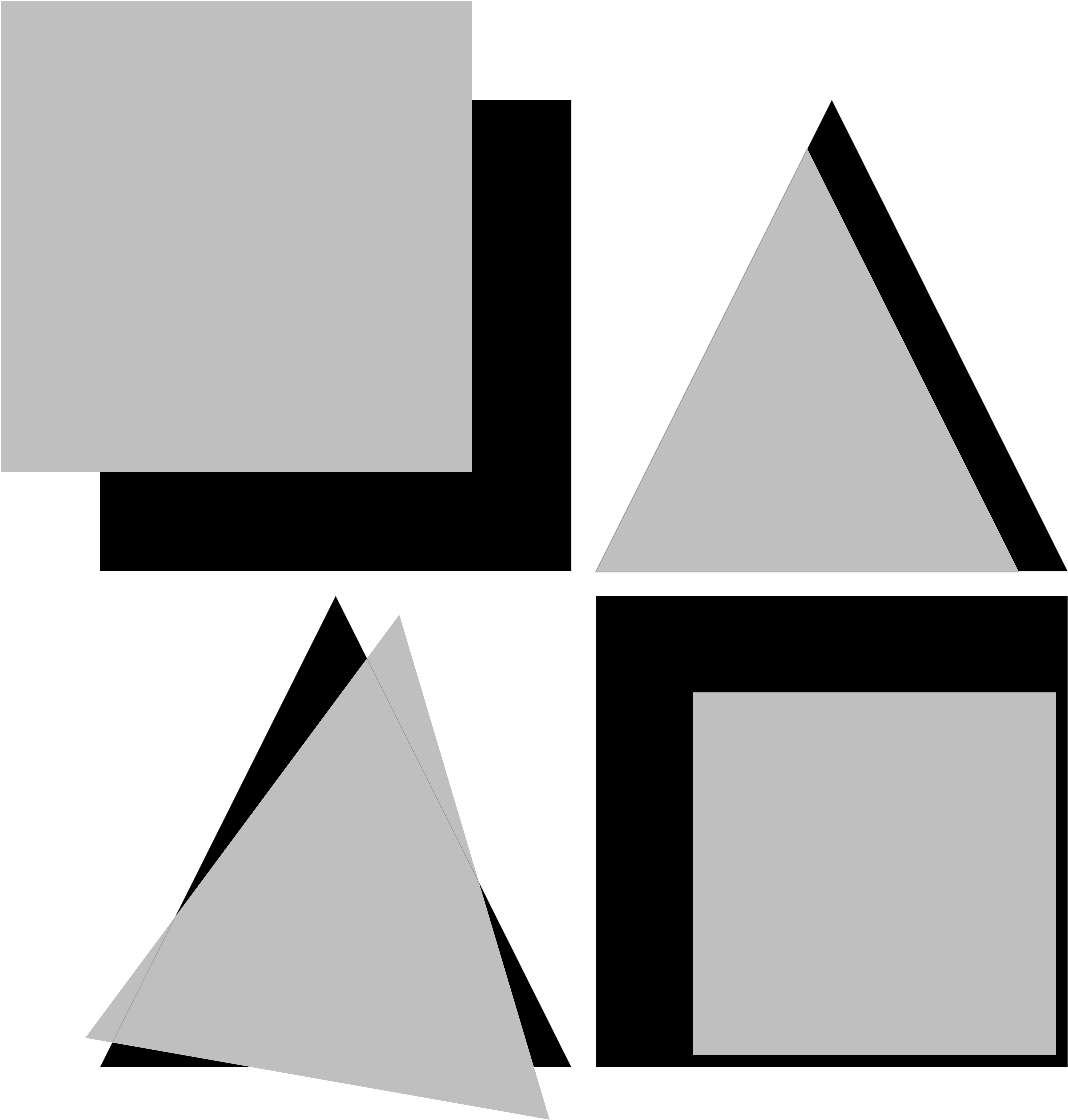} & \includegraphics[scale=\scalefactor]{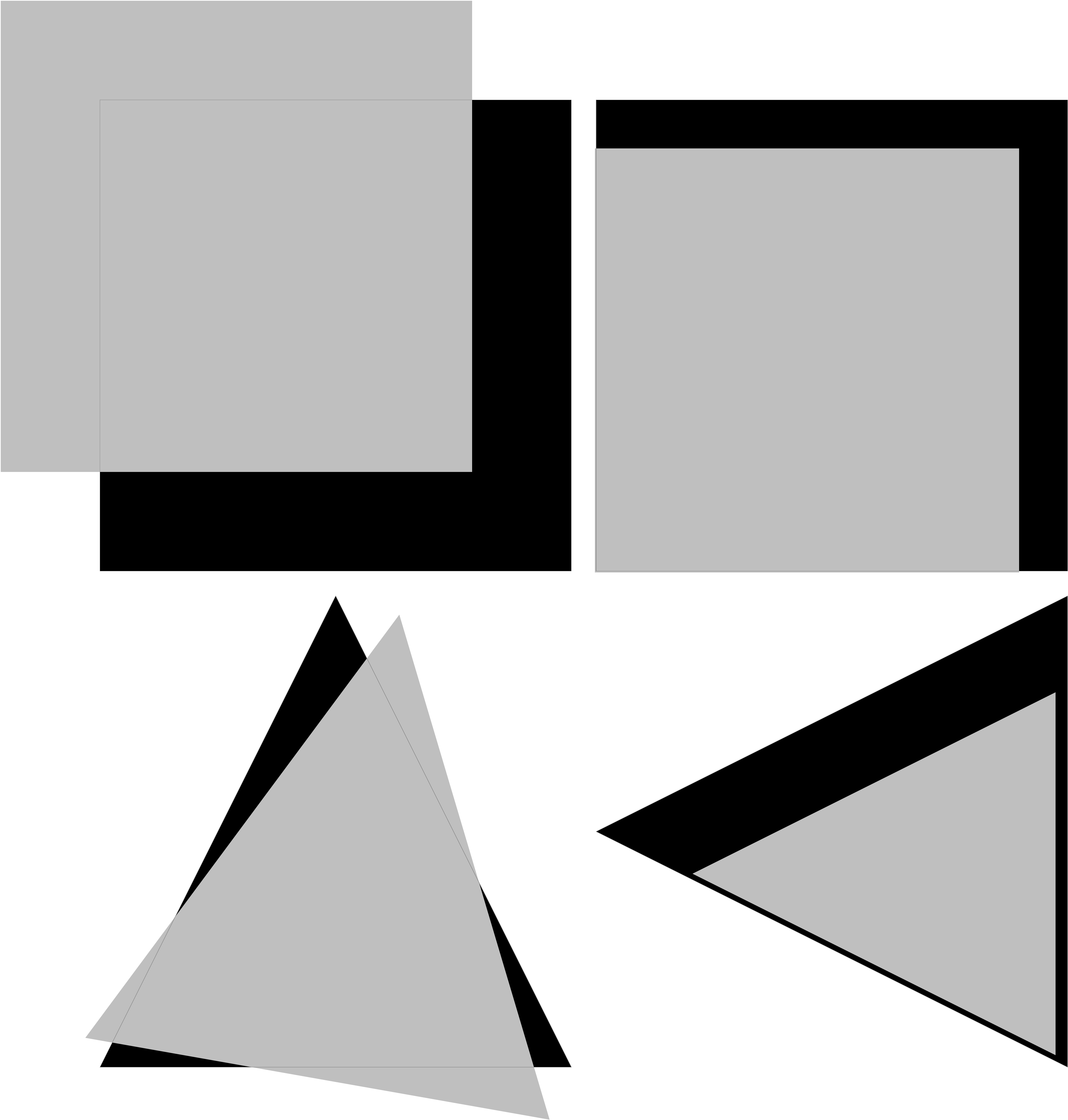} & \includegraphics[scale=\scalefactor]{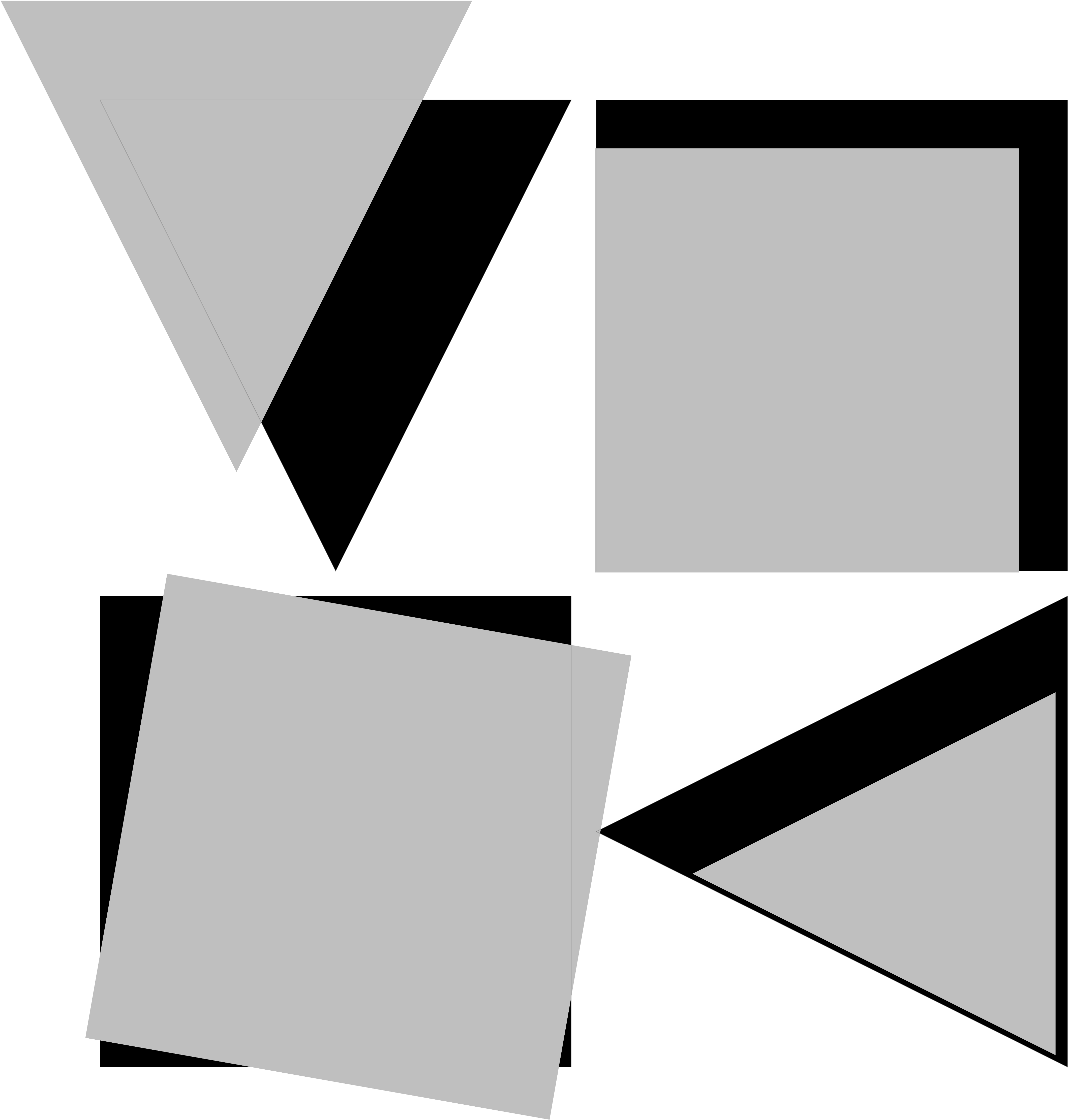} & \includegraphics[scale=\scalefactor]{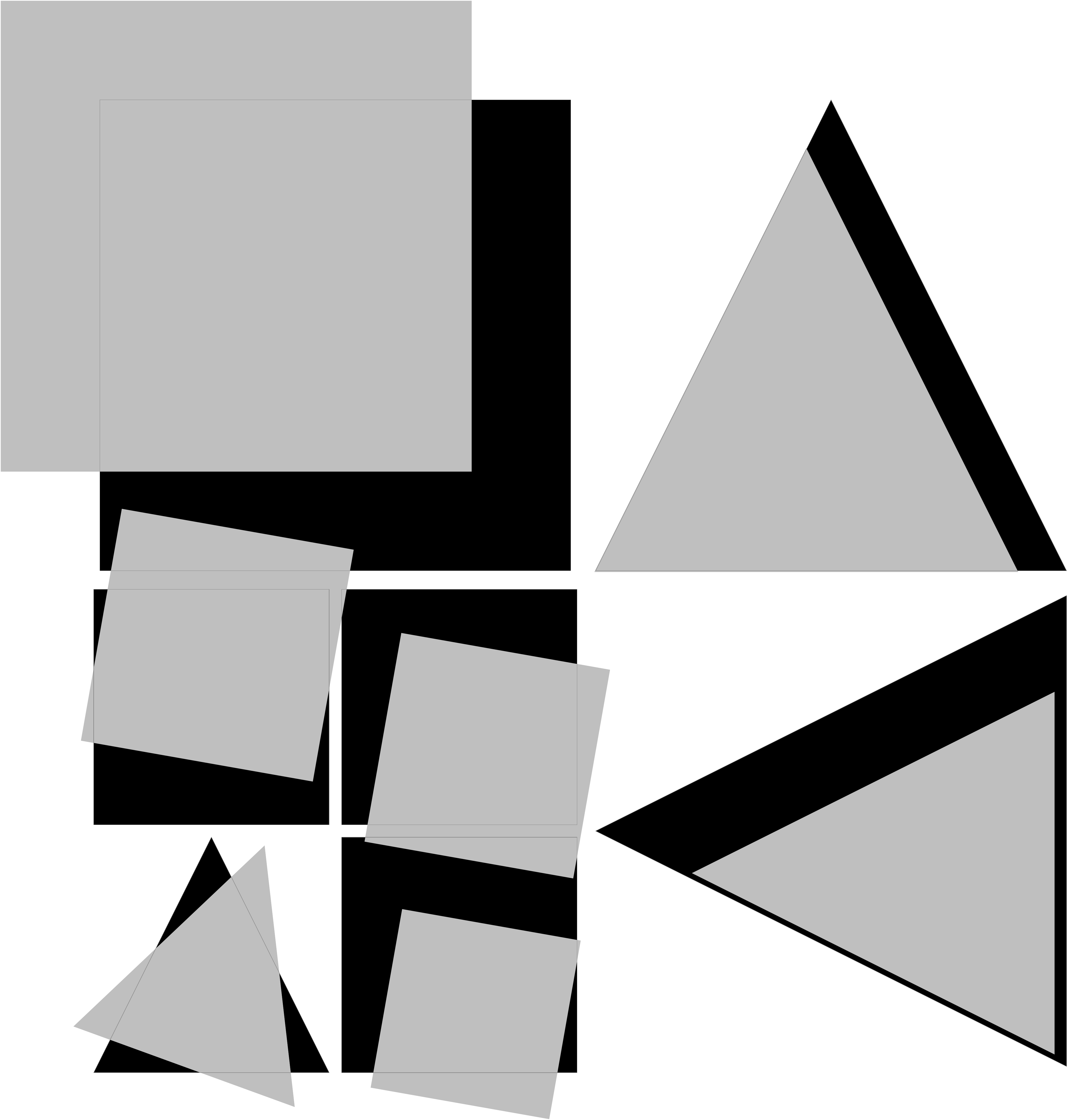}
\end{array}$%
}
  \caption{Top: four examples, where the first component of each example $(\delta(x_1, t), \val_{\calg}(t))$ has been shortened to $t$. Middle: overlay of each $\val_{\calg}(t)$ (black) and the corresponding $\val_{\calg_{\mathcal N}}(t)$ (grey) for some imperfect approximation of $F_\calg$ by an imaginary neural network $\mathcal N$. Bottom: similar overlay after a slight adjustment of the transformation parameters.\label{fi:collage examples}}
\end{figure}
As indicated in the second and third row of the figure, training~$\mathcal N$ should reduce the distance between $\val_{\calg}(t)$ and $\val_{\calg_{\mathcal N}}(t)$, ideally until the pictures match perfectly.
However, there is an interesting problem here for future research: For learning neural networks by back propagation, the loss function should be differentiable in each parameter of the network, i.e., in this case in each of the six parameters of each of the affine transformations involved. While the loss is easily seen to be continuous, it is not necessarily differentiable. Hence, it would be an interesting research question to find out whether this difficulty can be circumvented, for example by techniques such as those proposed by \citet{Grabocka.etAl:19} and \citet{Lee.etAl:20}. Even in cases where the loss is differentiable, another open problem is how to actually compute the partial derivatives. We leave these questions for future work because the purpose of the current paper is to propose the unified framework and motivate it by examples.

\subsection{Scene descriptions}
\label{sec:scene}

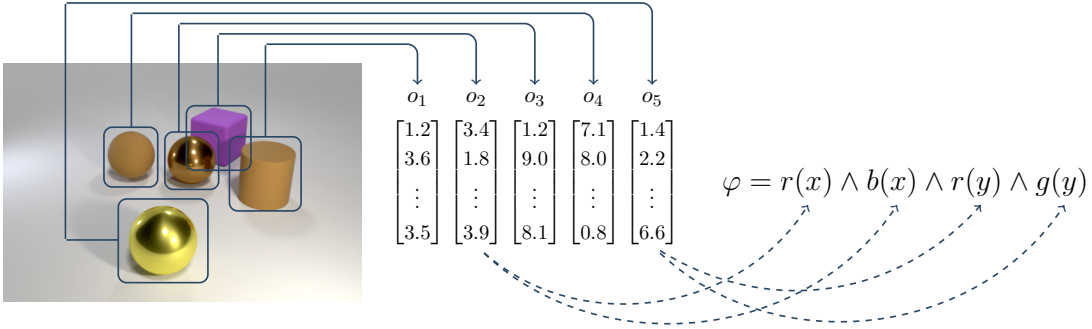
\begin{figure}[t]
    \centering
    \begin{tikzpicture}
\useasboundingbox (0,0) rectangle (2,5);
  \begin{scope}[transform canvas={scale=.78}, shift={(-8cm,0.75cm)}]
    \node[anchor=south west,inner sep=0] (image) at (0,0) {\includegraphics[width=0.4\textwidth]{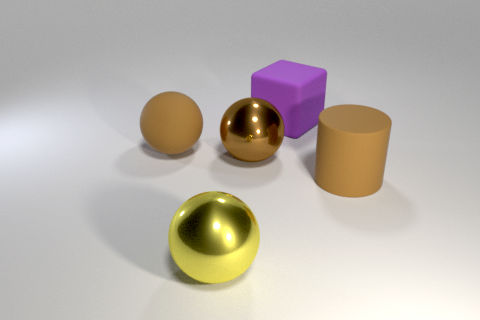}};
    \begin{scope}[x={(image.south east)},y={(image.north west)}]
        \draw[UmUBlue, thick,rounded corners] (0.32,0.08) rectangle (0.57,0.43);
        \draw[UmUBlue, thick,rounded corners] (0.28,0.48) rectangle (0.43,0.73);
        \draw[UmUBlue, thick,rounded corners] (0.45,0.47) rectangle (0.60,0.71);
        \draw[UmUBlue, thick,rounded corners] (0.51,0.55) rectangle (0.69,0.82);
        \draw[UmUBlue, thick,rounded corners] (0.63,0.38) rectangle (0.83,0.69);
    \end{scope}
\node [x={(image.south west)},y={(image.north west)}, xshift=56pt,yshift=30pt,inner sep=0, outer sep=0] (2,3) (obj5a) {};   
\node [x={(image.south west)},y={(image.north west)}, xshift=30pt,yshift=30pt,inner sep=0, outer sep=0] (2,3) (obj5b) {};    
\node [x={(image.south west)},y={(image.north west)}, xshift=30pt,yshift=145pt,inner sep=0, outer sep=0] (2,3) (obj5c) {};   
\node [x={(image.south west)},y={(image.north west)}, xshift=62pt,yshift=84pt,inner sep=0, outer sep=0] (2,3) (obj4a) {};   
\node [x={(image.south west)},y={(image.north west)}, xshift=62pt,yshift=140pt,inner sep=0, outer sep=0] (2,3) (obj4b) {};    
\node [x={(image.south west)},y={(image.north west)}, xshift=85pt,yshift=82pt,inner sep=0, outer sep=0] (2,3) (obj3a) {};   
\node [x={(image.south west)},y={(image.north west)}, xshift=85pt,yshift=135pt,inner sep=0, outer sep=0] (2,3) (obj3b) {};    
\node [x={(image.south west)},y={(image.north west)}, xshift=104pt,yshift=95pt,inner sep=0, outer sep=0] (2,3) (obj2a) {};   
\node [x={(image.south west)},y={(image.north west)}, xshift=104pt,yshift=130pt,inner sep=0, outer sep=0] (2,3) (obj2b) {};    
\node [x={(image.south west)},y={(image.north west)}, xshift=127pt,yshift=80pt,inner sep=0, outer sep=0] (2,3) (obj1a) {};   
\node [x={(image.south west)},y={(image.north west)}, xshift=127pt,yshift=125pt,inner sep=0, outer sep=0] (2,3) (obj1b) {};   

\node [right of=image,xshift=3cm] (v1) {\small \textcolor{black}{$\begin{bmatrix}
1.2\\[.5ex]
3.6\\[.5ex]
\vdots\\[.7ex]
3.5
\end{bmatrix}$}}; 
\node [above of=v1,yshift=2.5ex] (o1) {$o_1$};

\node [right of=v1] (v2) {\small$\begin{bmatrix}
3.4\\[.5ex]
1.8\\[.5ex]
\vdots\\[.7ex]
3.9
\end{bmatrix}$}; 
\node [above of=v2,yshift=2.5ex] (o2) {$o_2$};
\node [right of=v2]  (v3) {\small$\begin{bmatrix}
1.2\\[.5ex]
9.0\\[.5ex]
\vdots\\[.7ex]
8.1
\end{bmatrix}$};
\node [above of=v3,yshift=2.5ex] (o3) {$o_3$};
\node [right of=v3] (v4) {\small$\begin{bmatrix}
7.1\\[.5ex]
8.0\\[.5ex]
\vdots\\[.7ex]
0.8
\end{bmatrix}$}; 
\node [above of=v4,yshift=2.5ex] (o4) {$o_4$};
\node [right of=v4] (v5) {\small$\begin{bmatrix}
1.4\\[.5ex]
2.2\\[.5ex]
\vdots\\[.7ex]
6.6
\end{bmatrix}$}; 
\node [above of=v5,yshift=2.5ex] (o5) {$o_5$};
\draw[UmUBlue, rounded corners=5pt, thick,->] (obj5a)  -- (obj5b)  -- (obj5c) -| (o5) ; 
\draw[UmUBlue, rounded corners=5pt, thick,->] (obj4a)  -- (obj4b)  -| (o4) ; 
\draw[UmUBlue, rounded corners=5pt, thick,->] (obj3a)  -- (obj3b)  -| (o3) ; 
\draw[UmUBlue, rounded corners=5pt, thick,->] (obj2a)  -- (obj2b)  -| (o2) ; 
\draw[UmUBlue, rounded corners=5pt, thick,->] (obj1a)  -- (obj1b)  -| (o1) ; 
\node [below of=v2,yshift=-.5ex] (d2) {};
\node [below of=v5,yshift=-.5ex] (d5) {};

\node [right of=v5, xshift=20ex] (phi) {\Large $\varphi = r(x) \land b(x) \land r(y) \land g(y)$};
\node [below of=phi,shift={(-9.5ex,4.5ex)}] (dx1) {};
\node [below of=phi,shift={(-0.3ex,4.5ex)}] (dx2) {};
\node [below of=phi,shift={(8.5ex,4.5ex)}] (dy1) {};
\node [below of=phi,shift={(16.9ex,4.5ex)}] (dy2) {};

\draw[UmUBlue,thick, dashed,->,bend right=50] (d2) edge (dx1);
\draw[UmUBlue,thick, dashed,->,bend right=50] (d2) edge (dx2);
\draw[UmUBlue,thick, dashed,->,bend right=40] (d5) edge (dy1);
\draw[UmUBlue,thick, dashed,->,bend right=50] (d5) edge (dy2);

   \end{scope}
\end{tikzpicture}
    \caption{To predicates in a scene, the scene is first segmented into a set of objects, each of which is translated into a real-valued vector $o$. If the neural networks that realise the predicates of $\varphi$ have been satisfactorily trained, then there is an assignment of objects to the variables that satisfies $\varphi$. In this example, the predicates $r$, $b$, and $g$ can be understood as the properties of being round, being bronze coloured, and being gold coloured, respectively. Under this interpretation, the assignment of $o_2$ to $x$ and $o_5$ to $y$ satisfies $\varphi$.}
    \label{fig:clevr}
\end{figure}


  The final instantiation of the algebraic framework serves to learn (predicates representing) properties of physical objects,  so as to be able to ground the variables of a logic scene description in a scene. For simplicity, we limit the discussion to the type of synthetic scenes provided by~\cite{JohnsonEtAl:2017} through the CLEVR dataset. In this dataset, a scene is a configuration of a limited number of geometrical objects of different shapes, colors, materials, and sizes. For the present, we also disregard the relative positioning of objects within the scene, even though such relations could be modelled by an extension of the logic described in the following. A final simplification is the assumption that the scene has already been segmented and every object found has been translated into a real-valued vector $o$ by a pretrained visual embedding model. A scene is thus represented as a multiset of such vectors, and these multisets take the place of the objects in examples. 

  An example $(\varphi, O)$ consists of a multiset of object vectors $O$ and a term $\varphi$ representing a logic formula over unary predicates, logical connectives, and free variables. The predicates represent properties of objects, and the example provides an example of a true property (expressed by $\varphi$) of the scene object $O$, for some assignment of scene objects to the variables. An example is shown in Figure~\ref{fig:clevr}, where the objects in the scene have been translated into vector-based embeddings. The goal is to use the examples in order to discover the meaning of the predicates. 

  Let us now formalise this setting in our framework. For this, let $\Gamma = \{\alpha, \beta\}$ and $\Sigma = C \cup P$ with 
  \[C = 
  \{\land \colon \beta \times \beta \to \beta, \; 
  \lor \colon \beta \times \beta \to \beta, \; 
  \rightarrow \colon \beta \times \beta \to \beta,\; 
  \lnot \colon \beta \to \beta \}
  \]
  and  $P = \{p \colon \alpha \to \beta, \; \dots, \; p_n \colon \alpha \to \beta\}$. Moreover, $X$ is a set of variables, all of type $\alpha$. 
  Let $\alg = ((\dom_\gamma)_{\gamma \in \Gamma}, (f_\alg)_{f \in \Sigma})$, where $\dom_\alpha = \real^n$ and $\dom_\beta = [0,1]$, the latter denoting the unit interval of real numbers. Our evaluation domain is $\dom_\beta$ with the usual order on $[0,1]$, and $\bigoplus V=\sum V/|V|$ is the average of $V$ for every finite subset $V$ of $[0,1]$. The interpretation of the logical connectives $f_\alg$ for $f \in \Sigma \backslash P$ is given in Table~\ref{tab:logic}. In this template algebra, it is possible to use neural networks to model the to be learned predicates in $P$. These predicates map object embeddings to real numbers, where the value signals that the input object has a certain property (according to the neural network implementing the predicate), and a lower value signals that it does not. Since the learning goal is to maximise the total value $\bigoplus_{O\in\samp}\val_\alg(O)$ of all examples in the example set $\samp$, we can use $1-\val_\alg(O)$ as the loss function for training. Given a finite set $\samp$ of examples, the sum in Equation~\eqref{eq:value} can then be maximised by training a family of neural networks $(\theta_p)_{p \in P}$ to realise the set of predicates~$P$.

\begin{table}[tb]
\caption{Algebraic operations to model logic connectives for scene grounding\label{tab:logic}}
\smallskip
\centering
\begin{tabular}{r@{\;}c@{\;}l@{\qquad}r@{\;}c@{\;}l}
\toprule
$\land_\alg(n_1,n_2)$ & = & $\textrm{min}(n_1,n_2)$ & 
$\lor_\alg(n_1,n_2)$ & =  & $\textrm{max}(n_1,n_2)$\\
$\rightarrow_\alg(n_1,n_2)$ & = & $\textrm{max}(1-n_1,n_2)$ &
 $\lnot_\alg(n_1)$ & = & $1  -  n_1$ \\
 \bottomrule
\end{tabular}
\end{table}

The natural continuation of this theoretical work is an empirical study starting out from a more or less abstract representation of objects. For example, this could be a real-valued vector produced by applying an out-of-the-box autoencoder to the segmented object.
  One would then generate a set of training examples  $\{(\varphi_1, O_1),\dots,(\varphi_n, O_n)\}$ for some $n\in\nat$, and train the neural networks $(\theta_p)_{p \in P}$.

\section{Conclusion and future work}
\label{sec:future}
We have introduced an algebraic framework that can be instantiated for a variety of learning settings, allowing us to integrate grammatical inference and deep-learning approaches. As indicated throughout the paper, this is an initial theoretical effort, but framework lends itself well to empirical experiments, and this is also the the natural next step. One of the main advantages of the algebraic approach is that it allows us to look at a learning setting through a layer of abstraction. This makes it possible to treat various aspects in isolation, and opens for a wealth of research directions. 

In general, we may ask what properties of the algebraic domains and operations promote learning. We may also consider the effect of different valuation schemes, e.g., should we allow that two variables are assigned the same logic object, or require that every object is assigned to some variable. We may also study the effect of adding universal and existential quantifiers, and so requiring that certain properties hold for all or some of the elements in an example. Another question is the trade-off between computational complexity and expressiveness that comes from learning predicates of arity greater than one. Such an addition would allow us to learn relations, e.g., whether one object is smaller than another, but would doubtlessly also require greater computational effort. 

In the cases where we apply deep learning, it is natural to explore different forms of object representations and network architectures, but also the realisation of logic connectives in a numerical domain. Finally, we may also consider how we can manage a situation where the number of elements in an example is very large (consider, e.g., a painting of a city scene in great detail). Here, it may not be feasible to try every possible way of linking the variables in the accompanying formula to a segment of the scene, so one may devise various heuristics of limiting the number of combinations to evaluate. 

\bibliography{references}

\end{document}

%% file: defns.tex
\theoremheaderfont{\bfseries}
\theorempostheader{}
\theoremsep{\ }
\newtheorem{observation}{Observation}

\newcommand{\nat}{\mathbb{N}}
\newcommand{\real}{\mathbb{R}}

\makeatletter
\newcommand{\terms}[1][]{%
        \def\tempa{#1}%
        \def\tempb{}%
        \ifx\tempa\tempb\mathrm{T}_{\Sigma}%
        \else\mathrm{T}_{\Sigma\cup#1}%
        \fi%
}%
\makeatother
\DeclareMathOperator{\opt}{opt}
\newcommand{\domain}{\mathrm{dom}}
\newcommand{\lex}{<_{\mathrm{lex}}}

\newcommand{\prefixes}[1]{\mathit{prefixes}(#1)}
\newcommand{\pow}[1]{\ensuremath{2^{#1}}}  
\newcommand{\card}[1]{\left|#1\right|}

\newcommand{\typemapping}{\mathit{type}}
\newcommand{\types}{\ensuremath{\Gamma}\xspace}
\newcommand{\type}{\ensuremath{\gamma}\xspace}
\newcommand{\etype}{\ensuremath{\tau}\xspace}

\newcommand{\alg}{{\mathcal{A}}}

\newcommand{\dom}{{\mathbb{A}}}
\newcommand{\val}{\mathrm{val}}

\newcommand{\groundings}[2]{[#1\to #2]}
\newcommand{\samp}{\mathcal S}

\newcommand{\emptystr}{\varepsilon}
\newcommand{\powerset}[1]{\ensuremath{2^{#1}}}

\newcommand{\block}[2]{[#1]_{#2}} 
\newcommand{\partition}[2]{(#1/{#2})} 
\newcommand{\accept}{\mathit{accept}}
\newcommand{\equ}{\mathit{equiv}}
\newcommand{\simrc}{\sim^{\mathrm{rc}}}
\newcommand{\calg}{\ensuremath{\mathcal C}\xspace}
\newcommand{\collages}{\ensuremath{\mathbb C}\xspace}
\newcommand{\cop}[1]{\langle#1\rangle}

\definecolor{UmUBlue}{RGB}{42,71,101}
\definecolor{UmUGreen}{RGB}{115,167,144}
\definecolor{UmUGold}{RGB}{215,177,124}
\definecolor{UmUPink}{RGB}{234,186,185}
\definecolor{NoUmUColour}{RGB}{84,142,202}

\newcommand{\strings}[1]{\mathit{str}(#1)}
\newcommand{\pred}[1]{\mathit{pred}(#1)}
\newcommand{\aut}[1]{\textsc{Aut}(#1)}